\documentclass[leqno,11pt]{article}
\renewenvironment{description}[1][0pt]
  {\list{}{\labelwidth=0pt \leftmargin=#1
   }}
  {\endlist}

\usepackage[toc,page]{appendix}
\usepackage{enumerate,xspace}
\usepackage{amsmath,xspace,amssymb}
\usepackage{amsthm}
\usepackage{stmaryrd}
\usepackage{verbatim}
\usepackage{fullpage}
\usepackage{amsfonts}
\usepackage{proof}
\usepackage{calc}
\usepackage{tikz}
\usetikzlibrary{positioning}
\usetikzlibrary{decorations.markings}
\tikzstyle{vertex}=[circle, draw, inner sep=0pt, minimum size=6pt]
\newcommand{\vertex}{\node[vertex]}

\input xy
\xyoption{all}
\xyoption{2cell}
\UseAllTwocells
\CompileMatrices

\title{Another Generic Setting for Entity Resolution: Basic Theory}
\author{Xiuzhan Guo, Arthur Berrill,
Ajinkya Kulkarni,
Kostya Belezko,
Min Luo
}
\date{}

\newcommand{\rw}{\rightarrow}
\newcommand{\rst}[1]{\overline{#1}}

\newcommand{\two}{\ensuremath{{\hbox{\textrm 2}\kern-.25em
        \hbox{\vrule height1.5ex width 0.4pt depth -.2ex}}\kern.2em}\xspace}
\newcommand{\three}{\ensuremath{{\hbox{\textrm 3}\kern-.25em
        \hbox{\vrule height1.5ex width 0.4pt depth -.0ex}}\kern.2em}\xspace}

\newcommand\mdoubleplus{\mathbin{+\mkern-10mu+}}

\newtheorem{theorem}{Theorem}[section]    

\newtheorem{corollary}[theorem]{Corollary}   
\newtheorem{lemma}[theorem]{Lemma}   
\newtheorem{preremark}[theorem]{Remark}   
\newtheorem{prexample}[theorem]{Example}   
\newtheorem{proposition}[theorem]{Proposition}

\newtheorem{definition}[theorem]{Definition}

\newenvironment{remark}{\begin{preremark}\rm}{\end{preremark}}
\newenvironment{example}{\begin{prexample}\rm}{\end{prexample}}

\begin{document}

\maketitle

\begin{abstract}
Benjelloun et al. \cite{BGSWW} considered the Entity Resolution (ER) problem 
as the generic process of matching and merging entity records judged to represent the same 
real world object. They treated the functions for matching and merging entity records as black-boxes
and introduced four important properties (called Idempotence, Commutativity, Associativity, and Representativity) 
that enable efficient generic ER algorithms. 

A partial groupoid is a nonempty set $P$ with a partial operation $\circ$ that is
a partial function $P\times P\rw P$.
An ER system $(\mathbb{E},\approx, \langle\,\rangle)$ is the same as a partial groupoid
$(P,D,\circ)$, where the binary operation $\circ$ is given by 
the merge function $\langle\,\rangle$ and the partiality $D$ of the operation $\circ$ is specified 
by the match function $\approx$.
Given an instance $I$, the merge closure of $I$ turns out  to be the finitely generated partial subgroupoid $[I]$.

In this paper, we shall study the properties which match and merge functions share, 
model matching and merging black-boxes for ER in a partial groupoid,
based on the properties that match and merge functions satisfy,
and show that a partial groupoid provides another generic setting for ER.

The natural partial order on a partial
groupoid is defined  when the partial groupoid satisfies Idempotence and Catenary associativity.
Given a partial order on a partial groupoid, 
the least upper bound and compatibility ($LU_{pg}$ and $CP_{pg}$) properties are equivalent to
Idempotence, Commutativity, Associativity, and Representativity and the partial order must be the natural one we defined
when the domain of the partial operation is reflexive.
The partiality of a partial groupoid can be reduced using connected components and clique covers of its domain graph,
and a noncommutative partial groupoid can be mapped to a commutative one homomorphically
if it has the partial idempotent semigroup like structures.

In a finitely generated partial groupoid $(P,D,\circ)$ without any conditions required, 
the ER we concern is the full elements in $P$.
If $(P,D,\circ)$ satisfies Idempotence and Catenary associativity,
then the ER is the maximal elements in $P$,
which are full elements and form the ER defined in \cite{BGSWW}.
Furthermore, in the case, since there is a transitive binary order, we consider ER as
``sorting, selecting, and querying the elements in a finitely generated partial groupoid."
\end{abstract}



\pagenumbering{arabic}

\section{Introduction}\label{chapter:introduction}
Entity resolution (ER, also known as record linkage, or deduplication, or merge-purge) aims to decide if
two or more entity records in a database or stream represent the same real world
object. It is well known that ER is a crucial task to integrate data from different sources \cite{ccr} and one of the major
impediments affecting data quality in information systems and stream data analysis \cite{sw,cepps}.

Benjelloun et al. considered that in the ER problem, 
the records, determined to represent the same real world entity, are successively located and merged \cite{BGSWW}. 
They formalized the generic ER problem by treating the functions for matching and merging records as black-boxes, which permits expressive and extensible ER solutions. 
They introduced four important properties $\eqref{eqn:ei}$, $\eqref{eqn:ec}$, $\eqref{eqn:ea}$, and $\eqref{eqn:er}$ 
(Idempotence, Commutativity, Associativity, and Representativity, see below for the definitions) that, 
if satisfied by the match and merge functions, enable more efficient generic ER algorithms. 

Assume that $\mathbb{E}$ is a set of entity records. 
In the ER model of Benjelloun et al.  \cite{BGSWW},
a {\em match function} $\approx$ on $\mathbb{E}$ is a Boolean function 
$\approx: \mathbb{E}\times \mathbb{E}\rw \{0,1\}$,
where $\approx(e_1,e_2)=1$ means that $e_1$ matches $e_2$, denoted by $e_1\approx e_2$.
Otherwise, $\approx(e_1,e_2)=0$ means that $e_1$ does not match $e_2$, denoted by 
$e_1\not\approx e_2$.
A {\em merge function} $\langle\;\rangle$, with respect to the match function $\approx$, is a partial function 
$\langle\;\rangle:\mathbb{E}\times\mathbb{E}\rw \mathbb{E}$,
which is defined only on the pairs of matching entity records,
namely, $\langle\;\rangle$ is defined on 
$D\subseteq \mathbb{E}\times\mathbb{E}$, written by $\mbox{domain}(\langle\;\rangle)=D$,
where $(e_1,e_2)\in D$ if and only if $e_1\approx e_2$, and in this case the value $\langle\;\rangle(e_1,e_2)$ is denoted by $\langle e_1,e_2\rangle$.
An {\em ER system} $(\mathbb{E},\approx,\langle\;\rangle)$ consists of a nonempty set $\mathbb{E}$ of entity records,
a binary relation $\approx$ on $\mathbb{E}$ modeling the match function,
and a partial operation $\langle\;\rangle$ on $\mathbb{E}$ modeling the merge function.

The merge function is intended to merge the information of matching entity records together to generate new records
with more information.
As highlighted in \cite{BGSWW}, the merges can lead to new matches
and there might exist the scenarios in which the final match only happens after a number of matches and merges.

An {\em instance} $I$ is a finite subset $I\subseteq \mathbb{E}$.
The {\em merge closure }of an instance $I$, identifying all pairs of matching records in $I$ and merging them,
denoted by $\overline{I}$, is the smallest set $R$ of entity records such that
\begin{enumerate}[$(i)$]
\item
$I \subseteq R$, and
\item
$R$ is merge closed: for all $e_1,e_2\in R$, $e_1\approx e_2$ implies $\langle e_1,e_2\rangle \in R$.
\end{enumerate}

Since the intersection of merge closed subsets $\subseteq \mathbb{E}$ is also merge closed, we have
$$\overline{I}=\bigcap \Big\{R| I\subseteq R, R \mbox{ is merge closed}\Big\}.$$

Given an instance $I$, 
the goal of ER in \cite{BGSWW} was to select the smallest subset $ER(I)$ of the entity records in the merge closure
$\overline{I}$ such that $ER(I)$ has no less information than $\rst{I}$. 
To achieve this, the partial order 
$\leq$ was introduced in \cite{BGSWW}:
$$\mbox{Given two entity records $e_1,e_2\in\mathbb{E}$}, 
e_1\leq e_2\mbox{ if and only if } e_1\approx e_2 \mbox{ and }\langle e_1,e_2\rangle = e_2.$$

Given two sets of entity records
$I_1,I_2\subseteq \mathbb{E}$, $I_1\leq I_2$, called $I_1$ is {\em dominated} by $I_2$, if and only if
for each $e\in I_1$ there is $e'\in I_2$ such that $e\leq e'$.

In \cite{BGSWW}, entity resolution of an instance $I$ was defined as:
\begin{definition}[Definition 2.3 \cite{BGSWW}]
Given an instance $I\subseteq \mathbb{E}$,
{\em an entity resolution} $ER(I)$ of an instance $I$ is 
\begin{equation}\label{eqn:erswoosh}
ER(I)=\mbox{set $I'$ of entity records such that:}\tag{$ER$}
\end{equation}
\begin{enumerate}[$(i)$]
\item
$I'\subseteq \overline{I}$,
\item
$\overline{I}\leq I'$,
\item
$I'$ is minimal up to the conditions $(i)$ and $(ii)$: No strict subset of $I'$ satisfies the conditions $(i)$ and $(ii)$.
\end{enumerate}
\end{definition}

In  \cite{BGSWW}, Benjelloun et al. proved that the entity resolution of
an instance $I$ exists and is unique.

The definitions of $\eqref{eqn:ei}$, $\eqref{eqn:ec}$, $\eqref{eqn:ea}$, and $\eqref{eqn:er}$ in \cite{BGSWW} are as follows:
\begin{description}
\item[{\bf\em Idempotence:}] 
\begin{equation}\label{eqn:ei}
\mbox{For all $e\in\mathbb{E}, e \approx e$ and $\langle e, e\rangle  = e$.}\tag{I} 
\end{equation}
\item[{\bf\em Commutativity:}] 
\begin{equation}\label{eqn:ec}
\mbox{For all $e_1, e_2\in\mathbb{E}, e_1 \approx e_2 $ if and only if $e_2 \approx e_1$, and if $e_1 \approx e_2$
then $\langle e_1, e_2\rangle = \langle e_2, e_1\rangle$.}\tag{C}
\end{equation}
\item[{\bf\em Associativity:}] 
\begin{equation}\label{eqn:ea}
\mbox{For all $e_1 , e_2 , e_3\in\mathbb{E}$, if $\langle e_1 , \langle e_2 , e_3 \rangle\rangle$ 
and $\langle\langle e_1 , e_2\rangle, e_3\rangle$ exist, then }
\langle e_1 , \langle e_2 , e_3 \rangle\rangle=\langle\langle e_1 , e_2\rangle, e_3\rangle.\tag{A}
\end{equation}
\item[{\bf\em Representativity:}] 
\begin{equation}\label{eqn:er}
\mbox{If $e_3 = \langle e_1 , e_2 \rangle$, then for any $e_4$ such that  $e_1 \approx e_4$, $e_3 \approx e_4$.}\tag{R}
\end{equation}
\end{description}

In  \cite{BGSWW}, Benjelloun et al. developed three efficient ER algorithms: 
$R$-Swoosh that exploits $\eqref{eqn:ei}$, $\eqref{eqn:ec}$, $\eqref{eqn:ea}$, and $\eqref{eqn:er}$ properties and performs the comparisons at the granularity records 
(Alg. 3, Proposition 3.3 \cite{BGSWW}),
$F$-Swoosh that exploits $\eqref{eqn:ei}$, $\eqref{eqn:ec}$, $\eqref{eqn:ea}$, and $\eqref{eqn:er}$ properties,
uses feature-level comparisons, and can be more efficient than $R$-Swoosh by avoiding repeated feature comparisons 
(Alg. 4, Proposition 4.1 \cite{BGSWW}),
and the most generic ER algorithm 
$G$-Swoosh that does not require
$\eqref{eqn:ei}$, $\eqref{eqn:ec}$, $\eqref{eqn:ea}$,  and $\eqref{eqn:er}$ properties 
(Algs 1 and 2, Proposition 3.2 \cite{BGSWW} and Theorem 3.1 \cite{BGSWW}).
The interested reader can consult \cite{BGSWW} for the algorithm details.

However, entity records might not be clients and can be other things, such as, data sets, ontologies, documents, graphs, etc. 
Merge functions might be other operations related,
e.g., data set join, ontology merge, document merge, graph merge, etc., respectively. 
The processes of matching and merging are very common,
not only in client record linkage search, but also in data set relationship detection, ontology alignment and merging,
document match and merge, graph similarity search and merge, etc.
It is natural to study the core properties that the match and merge operations share first
and then ask if there is a flexible setting in which the common practices mentioned above 
can be modeled together without any internal details being mentioned.

Algebraic structures, such as, groupoids, monoids, semigroups, groups, rings, etc., have a set of objects and one or more operations. The entity records concerned form a nonempty set and
a merge function can be viewed as a binary operation on the nonempty set.
We now recall the notions of groupoids and semigroups.

A {\em groupoid} is a nonempty set $P$ with a binary operation that is a function from the Cartesian product $P\times P$ to $P$.
A {\em semigroup}, another algebra structure, is a nonempty set $T$ together with an 
{\em associative} binary operation $\circ$ on $T$,
i.e., $\circ$ can be performed regardless of how objects from $T$ are grouped or where the parentheses are put. 
Semigroups have been of particular importance in computer science since the 1950s, 
such as, automatic semigroups
(a finitely generated semigroup equipped with several regular languages over an alphabet representing a generating set \cite{crrt}), 
graph inverse semigroups \cite{ah, MM}, etc.

Some binary operations on a nonempty set $P$ are defined only on a subset of $P\times P$. For example, left\_join in a structured query language is defined on only
the tables that have certain join conditions and a matrix multiplication $M\cdot N$ is defined only when the column number of $M$ is equal to the row number of $N$.
While a nonempty set $P$ with a partial binary operation, a partial function $\circ: P\times P\rw P$, 
is called a {\em partial groupoid}, a nonempty set, with an associative partial binary operation,
is called a {\em partial semigroup}.

The set $\mathbb{E}$ of entity records, along with match function $\approx$ and merge function $\langle\,\rangle$, 
can be viewed as a partial groupoid or semigroup (if $\langle\,\rangle$ is associative) $(\mathbb{E},\approx,\langle\,\rangle)$,
where $\langle\,\rangle$ is a partial binary operation on $\mathbb{E}$ and $\approx$ specifies the partiality of $\langle\,\rangle$. 
Hence match function and merge function in an ER system can be studied together in a partial
groupoid and therefore the related algebraic results (e.g., partial orders and finiteness conditions on semigroups) 
can be applied to ER systems.

To compare the information in entity records,
a partial order was introduced in \cite{BGSWW}.
It is well known that each semigroup can be partially ordered naturally \cite{M}.

Given a finite subset $I$ of entity records, both match function and merge function are applied to $I$
to generate all possible new entity records.
The merge closure $\rst{I}$ of $I$ turns out to be $[I]$, the partial subgroupoid,
generated by $I$, in $\mathbb{E}$, and so
the ER of $I$ turns out to find the specific elements in the finitely generated partial groupoid $[I]$.

To compute matches and merges effectively, it is desired that the steps of matching and merging is finite or the merge closure 
$\overline{I}$ of each finite instance $I$ is finite,
which means that the partial subgroupoid $[I]$ generated by $I$ is finite.
In order to force $[I]$ or $\overline{I}$ to be finite, some conditions are required.
The associativity property, if satisfied, the order in which records are operated not impacting on the final result, can reduce
the number of entity records generated from $I$ significantly. 
A partial groupoid comes to be a partial semigroup if it satisfies an associativity property.

The objectives of this paper are 
\begin{itemize}
\item
to model an entity resolution system as a partial groupoid and study the core properties on the partial groupoid, 
which match and merge operations share,
\item
to study partial orders on a partial groupoid and their interactions with the operations of the partial groupoid,
\item
to provide the ways to reduce partiality of a merge operation and to map a noncommutative partial groupoid that satisfies
certain reasonable conditions, to a commutative one by a partial groupoid homomorphism, and 
\item
to show that partial groupoids, when satisfying certain reasonable conditions, provide another generic setting for modeling
match and merge operations and computing ER systems.
\end{itemize}
The applications of the theory shall be summarized in the sequels to the paper.

The rest of the paper is organized as follows:
First, in Section \ref{chapter:partialsemigroups}, we recall the basic definitions and notations of a groupoid, such as, 
partial groupoid, semigroup,  and partial semigroup,
product of the subsets,
partial groupoid homomorphism, partial groupoid property preserved under partial groupoid homomorphisms,
irreducible generating set
and provide some examples from the different areas to show match and merge operations are very common.

In Section \ref{sect:properties}, we start with the partial groupoid properties
$\eqref{eqn:i}$, $\eqref{eqn:c}$, $\eqref{eqn:a}$, $\eqref{eqn:rl}$, $\eqref{eqn:rr}$, 
and $\eqref{eqn:r}$.
To define a partial semigroup, one must specify one version associativity from many associativity properties.
In this paper, we are interested in the three associative properties: 
$\eqref{eqn:a}$, $\eqref{eqn:cass}$, and $\eqref{eqn:sass}$.
Then we discuss the relationships between these properties,
e.g.,  $\eqref{eqn:cass}$ is equivalent to both $\eqref{eqn:a}$ and $\eqref{eqn:r}$ (Proposition $\ref{prop:ca_a_r}$),
either $\eqref{eqn:sass}$ or $\eqref{eqn:cass}$ implies $\eqref{eqn:a}$ (Proposition $\ref{prop:associativities}$),
these properties are preserved under partial groupoid homomorphisms (Proposition $\ref{prop:preserved}$).

A partial order is a  {\em reflexive, antisymmetric, and transitive} binary relation.
Natural partial orders on a semigroup
have been studied extensively \cite{M}.
Following the approach on semigroups, in Section \ref{sect:posets}, we define the natural partial orders
$\leq_l$, $\leq_r$, and $\leq$ and the maximal elements with respect to the
partial orders on a partial groupoid, respectively.
The maximal elements form the $ER(I)$ of a given instance $I$, defined in \cite{BGSWW}.
Given a partial order on a partial groupoid, to study the interactions 
between the partial order and the partial operation,
we consider the properties $\eqref{eqn:lub}$ and $\eqref{eqn:comp}$
with respect to the partial order and
prove that $\eqref{eqn:lub}$ and $\eqref{eqn:comp}$ are equivalent to
$\eqref{eqn:i}$, $\eqref{eqn:wc}$, $\eqref{eqn:a}$, 
and $\eqref{eqn:r}$ and the partial order is the natural one we defined,
when the domain of the partial operation is reflexive
(Theorem $\ref{thm:partialordersequal}$).

In general, match and non-match labels for entity resolution is highly imbalanced, 
called the {\em class imbalance problem} \cite{FCW, SWL}. 
Hence, given an ER system $(\mathbb{E},\approx,\langle\,\rangle)$ and an entity record $e\in \mathbb{E}$,
$e$ can compose only with a small part of the set $\mathbb{E}$, reflected by the partiality of $\langle\,\rangle$.
From the algebraic view, the partiality of the partial binary operation $\circ$ of a given partial groupoid $(P,D,\circ)$ 
can be measured by its domain $D$, which can be represented by a directed graph 
${\mathcal G}_P$. In Section \ref{sect:domaingraph}, we reduce the partiality of a partial groupoid $(P,D,\circ)$ 
using {\em connected components} and {\em clique covers} of $(P,D,\circ)$.
 
Commutativity property $\eqref{eqn:ec}$ is very important in \cite{BGSWW}. 
In Section \ref{sect:commutativegroupoid}, we show that each partial groupoid, having  idempotent semigroup-like structures, can be converted to a commutative partial groupoid homomorphically.

Since the merge closure $\rst{I}$ of a given instance $I$ can be viewed as a finitely generated groupoid $P$, 
after discussing the core properties and partial orders on a partial groupoid, we are ready to
present entity resolution on a given finitely generated partial groupoid $P$.
To be generic, in Section \ref{chapter:erongroupoids}, entity resolution on a finitely generated partial groupoid is studied, based
on the properties the partial groupoid satisfies.
We start with an arbitrary partial groupoid. In the case, the entity resolution $\eqref{eqn:er0}$ 
can be considered as the {\em full elements} in the partial groupoid.
If a given partial groupoid satisfies $\eqref{eqn:i}$ and $\eqref{eqn:cass}$, then it has a natural order $\leq$
and its entity resolution $\eqref{eqn:er1}$ 
can be considered as its maximal elements, which form the $\eqref{eqn:erswoosh}$ defined in \cite{BGSWW}.
Furthermore, if a finitely generated partial groupoid has a transitive order and is finite, then its entity resolution 
$\eqref{eqn:erpg}$ can be viewed as
{\em sorting, selecting, and querying the elements in a finitely generated partial groupoid}, based on the results 
for a transitive partial order in \cite{dkmrv}.
Finally, we give our concluding remarks.

\section{Preliminaries}\label{chapter:partialsemigroups}
Let $P$ be a nonempty set. A {\em binary operation} $\circ$ on $P$ is a function $\circ: P\times P \rw P$.
A {\em partial binary operation} $\circ$ on $P$ is a partial function $\circ: P\times P \rw P$, 
defined only on a subset $D\subseteq P\times P$.
$D$ is the {\em domain} of the partial operation $\circ$, denoted by $\mbox{domain}(\circ)=D$.
If $(x,y)\in D$, then we say $\circ (x,y)$ is {\em defined}, or $x$ is {\em composable} with $y$,
or $x\circ y$ exists.
The value $\circ(x,y)$ or the composition of $x$ and $y$ is denoted by $x\circ y$.
If $(x,y)\notin D$, then $\circ(x,y)$ is {\em undefined}, written by $x\circ y=\;\uparrow$.

Domain $D$ is {\em reflexive} if $(p,p)\in D$ for all $p\in P$.
$D$ is {\em symmetric} if
\begin{equation}\label{eqn:sym}
(x,y)\in D \Leftrightarrow (y,x)\in D.\tag{$S$}
\end{equation}
A nonempty set $P$, together with a (partial) binary operation $\circ$ on $P$, is called a {\em (partial) groupoid}. 
To emphasize the partiality of the operation $\circ$,
we write a {\em partial groupoid} as a triple $(P, D, \circ)$ or $(P,\sim,\circ)$,
where $P$ is nonempty set, $\circ$ is a partial binary operation on $P$, $D$ is the domain of $\circ$,
and $\sim$ is a binary relation on $P$ specifying $D$: $(x,y)\in D$ if and only if $x\sim y$.
A partial groupoid is also called an {\em magma} \cite{H}. 

Recall that an {\em undirected graph} or a {\em graph} 
consists of a pair $(V, E)$, where $V$ is a (finite) set of objects, called {\em vertices} or {\em nodes},
and $E$ is a set of unordered pairs $\{v_1,v_2\}$, called {\em edges}, $v_1,v_2\in V$.
A {\em directed graph} is a pair $(V,E)$, where $V$ is a (finite) set of objects, called {\em vertices} or {\em nodes},
and $E$ is a set of ordered pairs $(v_1,v_2)$, called {\em edges}, $v_1,v_2\in V$. 
In a directed graph, an edge $e=(v_1,v_2)$ can be viewed as an {\em arrow}, leaving $v_1$ and ending $v_2$, denoted by
$d_0(e)=v_1$ and $d_1(e)=v_2$.

It is well known that there are classical ways to convert between directed graphs and undirected graphs.

Throughout this paper, $(P,D,\circ)$ is a partial groupoid.
We may abbreviate ``$(P,D,\circ)$"  by $P$ or $(P,\circ)$ and omit $\circ$ in ``$x\circ y$" and write as ``$xy$"
when $D$ and $\circ$ are clear from the context.

\begin{example}\label{examp:pargroup}
Here are some examples of groupoids and partial groupoids.
\begin{enumerate}[$1.$]
\item\label{xamp:pargroup1}
Let $a,b,c$ be distinct elements and $P_1=\{a,b,c\}$. Define 
$$\begin{array}{ccccc}
a\cdot a = a, & b\cdot b = b, & c\cdot c= c,&
 a\cdot b=b, & b\cdot c =c,
\end{array}$$ 
and no other compositions are defined.
Then $(P_1,\cdot)$ is a partial groupoid.
\item\label{xamp:pargroup_naturalnumber}
Let $\mathbb{N}$ be the set of all natural numbers. Define
$$m\vee n= \max(m,n)\mbox{ and }m\wedge n=\min(m,n).$$
Then $(\mathbb{N},\vee)$ and $(\mathbb{N},\wedge)$ are groupoids.
\item\label{xamp:pargroup2}
Each ER system $(\mathbb{E},\approx,\langle\;\rangle)$ is a partial groupoid clearly, where
\[ \langle\;\rangle (x,y)=\left\{ \begin{array}{rl} 
                           \langle x,y\rangle, & \mbox{whenever $x\approx y$}\\
                           \uparrow, & \mbox{otherwise.}
                           \end{array}\right. \]

\item\label{xamp:pargroup3}
Let $A=\{a_1,\cdots,a_n,\cdots\}$ be a set of infinitely many entity objects. Define 
$a_i*a_{i+1}=a_{i+2}$
for all positive integers $i$
and no other compositions are defined.
Then $(A,*)$ is a partial groupoid.

\item\label{xamp:path}
Given a directed graph $G=(V(G),E(G))$, a {\em path} in $G$ is  a sequence of edges directed in the same direction which joins a sequence of vertices such that both edges and vertices are distinct. 
More formally, a path in $G$ is given by $[e_1,\cdots,e_m]$, where $e_1,\cdots, e_m\in E(G)$
such that $d_1(e_i)=d_0(e_{i+1})$ for $i=1,\cdots, m-1$, $|\{e_1,\cdots,e_m\}|=m$,
and $|\{d_1(e_1),\cdots,d_1(e_m)\}|=m$, where $d_1(e)=v_2$ and $d_0(e)=v_1$ for each edge $e:v_1\rw v_2$. 

Given two paths $[e_1,\cdots,e_m],[f_1,\cdots, f_n]$,
define $[e_1,\cdots,e_m]\sim_p [f_1,\cdots,f_n]$ if and only if there exists a smallest positive integer $i\leq \min\{m,n\}$ such that 
$$m-i+1\leq n,\;
e_i=f_1,\cdots,e_{m}=f_{m-i+1}.$$
Let ${\cal P}ath(G)$ be the set of all paths in $G$.
The composition of paths $[e_1,\cdots,e_m]$ and $[f_1,\cdots,f_n]$ is defined by
\[ [e_1,\cdots,e_m]\between [f_1,\cdots,f_n]=\left\{ \begin{array}{rl} 
                           [e_1,\cdots,e_i,f_{i+1},\cdots, f_n], & \mbox{if\;\; $[e_1,\cdots,e_m]\sim_p [f_1,\cdots,f_n]$}\\
                           \uparrow, & \mbox{otherwise.}
                           \end{array}\right. \]
Then $({\cal P}ath (G), \sim_p, \between)$ is a partial groupoid.

For example, let $W=\{p_1,\cdots,p_n\}$ be a set of webpages. A {\em click graph} ${\mathcal C}_W$
is the graph with 
$W$ as its vertices. There is an edge from $p_i$ to $p_j$ if a visitor clicks on $p_j$ after viewing $p_i$. Then 
$({\cal P}ath({\mathcal C}_W), \sim_p,\between)$ is a partial groupoid and its objects are 
clickstreams or click paths.  
\item\label{xamp:pargroup6}
Let $T$ be the set of tables in a relational database and $x$ be a column selected.
For all $t_1,t_2\in T$, define
\[ t_1\Join_l t_2=\left\{ \begin{array}{rl}
                          t_1 \;left\_join\; t_2\;on\;t_1\cdot x=t_2\cdot x, & \mbox{if $x$ exists in both $t_1$ and $t_2$}\\
                           \uparrow, & \mbox{otherwise.}
                          \end{array}\right. \]
and
\[ t_1\;\& \;t_2=\left\{ \begin{array}{rl}
    row\;bind\;of\;t_1\;and\;t_2,
     &  \mbox{if $t_1$ and $t_2$ are matched by their columns}\\
                          \uparrow, & \mbox{otherwise.}
                           \end{array}\right. \]
Then both $(T,\Join_l)$ and $(T,\&)$ are partial groupoids.
\item\label{xamp:pargroup_datalineage}
Given a database ${\mathcal D}$ with a set ${\mathcal T}$ of database operations, such as, join $\bowtie$,
$\sigma$, $\alpha$, a {\em data lineage element} in ${\mathcal D}$
consists of a triplet $\langle I,T,O\rangle$, where $T\subseteq {\mathcal T}$ is a sequence of database operations,
$I$ is the set of inputs to $T$, and $O$ is the output after applying the operations in $T$ one by one to $I$.
Let ${\mathcal L}$ be the set of data lineage elements in ${\mathcal D}$. Define
$$
\langle I,T,O\rangle c\langle I',T',O'\rangle
= \left\{ \begin{array}{rl}
(I,T\mdoubleplus T',O'), & \mbox{if} \;\; O=I'\\
\uparrow, & \mbox{otherwise.}\\
\end{array}\right.
$$
Here $\mdoubleplus$ is the sequence concatenation operation.
Then $({\mathcal L}, \leadsto,c)$ is a partial groupoid,
where $(I,T,O)\leadsto(I',T',O')$ if $O=I'$.
\item\label{xamp:pargroupMont}
Ontology has applications in multiple domains to represent relationships among classes and their instances.
An ontology alignment expresses the relations between different
ontologies \cite{HKES}.
Merge of ontologies creates a new ontology by linking up the existing ones and can be captured
by the pushout construction \cite{ZKEH,gbkbl}. Let $\mathbb{O}$ be the category of ontologies and ontology homomorphisms between them.
A $V$-alignment is a {\em span} in $\mathbb{O}$:
$$\xymatrix{
O_1 && O_2\\
& O \ar[lu]^{\rho_1} \ar[ru]_{\rho_2}
}$$
where $O, O_1, O_2$ are ontologies and $\rho_1:O\rw O_1$ and $\rho_2:O\rw O_2$ are ontology homomorphisms.
The pushout $\sqcup$ of a $V$-alignment is a simple description of an ontology merge \cite{ZKEH}, which can be viewed
as the {\em coproduct} in the {\em coslice} category $O/\mathbb{O}$
$(O/\mathbb{O},\,\mbox{v},\,\sqcup)$ is a partial groupoid,
where $O_1\,\mbox{v}\,O_2$ if there is a pair of ontology homomorphisms $\rho_1:O\rw O_1$ and $\rho_2:O\rw O_2$.

\item\label{xamp:mergeoverlap}
Given a set $E$ of entity objects, e.g., data sets, knowledge graphs, documents, etc., for $e_1,e_2\in E$,
define $e_1\smile e_2$ if and only if there is an {\em overlap} between $e_1$ and $e_2$
and $e_1\merge e_2$ is the union of $e_1$ and $e_2$ by identifying and keeping one copy of the overlaps.
Then $(E,\smile, \merge)$ is a partial groupoid.

\item\label{xamp:pargroup4}
Given a groupoid $(P,\cdot)$, each nonempty subset $S$ of $P$ is a partial groupoid:
for all $s_1,s_2\in S$, if $s_1\cdot s_2\notin S$ then $s_1\circ s_2$ is undefined. Otherwise,
$s_1\circ s_2=s_1\cdot s_2$. $(S,\circ)$ is a partial groupoid.
\end{enumerate}
\end{example}

A nonempty set $P$ may associate a set of partial operations.
For instance, a set $D$ of data tables in a database has join, combine, insert, update, Cartesian products, etc.

Let $(P,D,\circ)$ and $(Q,E, \cdot)$ be two partial groupoids. 
A {\em partial groupoid homomorphism} $f:(P,D,\circ)\rw (Q,E,\cdot)$
is a function $f:P\rw Q$ such that 
$$\xymatrix{
D\ar[d]_{\circ} \ar[r]^{f\times f} & E \ar[d]^{\cdot}\\
P\ar[r]^f& Q
}$$
commutes,
namely,
for all $x,y\in P$ if $x\circ y$ is defined then $f(x)\cdot f(y)$ is defined
and $f(x\circ y)=f(x)\cdot f(y)$.

The {\em image} of a partial groupoid homomorphism $f:(P,D,\circ)\rw (Q,E,\cdot)$, denoted by $\mbox{im} (f)$,
is defined to be the partial groupoid
$(f(P),f(D),\cdot)$, where $f(P)=\{f(p)\,|\,p\in P\}$ and $f(D)=\{(f(p),f(q))\,|\,(p,q)\in D\}$.
A property on a partial groupoid $(P, D,\circ)$ is 
{\em preserved or invariant under partial groupoid homomorphisms}
if for each partial groupoid homomorphism $f:(P,D,\circ)\rw (Q,E,\cdot)$
the partial groupoid $\mbox{im}(f)$ satisfies the property.

\begin{definition}
Given a partial groupoid $(P,\circ)$ and a positive integer $n$, the {\em product} of the subsets $I_1,\cdots,I_n$ of $P$ is defined 
as $I_1\cdots I_n\subseteq P$ inductively:
\begin{itemize}
\item
If $n=1$, $I_1\cdots I_n=I_1$.
\item
If $n>1$,
\[ x\in I_1\cdots I_n \mbox{ if and only if }  \left\{ \begin{array}{ll}
                           \exists \;postive \;integer\; k: 1\leq k< n\\
                           \exists \; y\in P:y\in I_1\cdots I_k \\
                           \exists \; z\in P: z\in I_{k+1}\cdots I_n\\
                           \end{array}\right\}x=yz \]
\end{itemize}
\end{definition}
Obviously, the product $I_1\cdots I_n$ is the set of all possible values 
$$x_1\cdots (x_ix_{i+1})\cdots x_n,$$ 
where $x_1\cdots x_i\cdots x_n$ are binarily grouped with the parentheses ``('' and ``)'' being inserted,
$x_i\in I_i, i=1,\cdots,n$.

$I_1\cdots I_n$ can be the singleton set $\{\uparrow\}$ if all possible ways of inserting ``(" and ``)" into 
$x_1\cdots x_i\cdots x_n$ always yield $\uparrow$
for all $x_i\in I_i, i=1,\cdots,n$.

If $I_1=\cdots =I_n=I$, write
$$I_1\cdots I_n=I^n.$$

If each $I_i=\{x_i\}$, write
$$I_1\cdots I_n=x_1\cdots x_n.$$

Throughout this paper, for $x_1,\cdots,x_n\in P$,
$x_1\cdots x_n$ stands for the product of $\{x_1\}\cdots \{x_n\}$.

Each nonempty subset $T$ of a partial groupoid $(P, \circ)$, together with the operation
$\circ'=\circ|_T$,
the restriction of $\circ$ to the set $T$, is a partial groupoid, called a {\em partial subgroupoid} of $(P,\circ)$.
If $T$ is closed with respect to $\circ$, i.e., $T^n\subseteq T$ for all positive integer $n$, then the partial subgroupoid 
$(T,\circ')$ is called a {\em closed partial  subgroupoid}.

Let $M\subseteq P$ and $T$ a closed partial subgroupoid of $(P,\circ)$. $(T,\circ)$ is {\em generated} by $M$ and $M$ is a {\em generating set} of $T$
if $M\subseteq T$ and $T$ is the smallest partial subgroupoid of $(P,\circ)$, containing $M$,
denoted by $T=[M]$. Obviously,
$$[M]=\bigcup_{i=1}^{+\infty}M^i.$$

The generating set $M$ of a partial groupoid $(P,\circ)$ is {\em irreducible} if there does not exist a proper subset of $M$, which is a generating set of $(P,\circ)$.

\begin{theorem}[{\bf 1}.4.9 \cite{LE}]\label{theorem:irreduciblegenset}
If a partial groupoid has a finite generating set $M$, 
then each generating set has an irreducible generating set contained in $M$.
\end{theorem}

By Theorem \ref{theorem:irreduciblegenset}, for a given finitely generated partial groupoid, we can assume its generating set is irreducible. 

\section{The Partial Groupoid Properties}\label{sect:properties}
By Example $\ref{examp:pargroup}.\ref{xamp:pargroup2}$,
each ER system $(\mathbb{E},\approx,\langle\;\rangle)$
is a partial groupoid.
Therefore, the properties on ER systems, introduced in \cite{BGSWW}, can be considered on partial groupoids.
We consider the following properties on a partial groupoid $(P,D,\circ)$:
\begin{description}
\item[\emph{Idempotence:}] 
For all $p\in P$, 
\begin{equation}\label{eqn:i}
\mbox{$(p,p)\in D$ and $p\circ p = p$.}\tag{$I_{pg}$} 
\end{equation}
Each element is always composable with itself, and the composition of each element with itself
yields the same element.
\item[\emph{Commutativity:}] 
For all $p_1, p_2\in P$, if $(p_1, p_2)\in D$ and $(p_2,p_1)\in D$, then  
\begin{equation}\label{eqn:wc}
p_1\circ p_2 = p_2\circ p_1.\tag{$C_{pg}$}
\end{equation}
Compositions
do not depend on the order of the elements being composed if the compositions exist.
\item[\emph{Strong commutativity:}]
For all $p_1, p_2\in P$,
\begin{equation}\label{eqn:c}
\mbox{$(p_1,p_2)\in D\Leftrightarrow (p_2,p_1)\in D$ and if $(p_1, p_2)\in D$ then }p_1\circ p_2 = p_2\circ p_1.\tag{$SC_{pg}$}
\end{equation}
$D$ is symmetric and compositions
do not depend on the order of the elements if the compositions exist.
\item[\emph{Left representativity:}] 
Let $p_1,p_2\in P$ such that $(p_1,p_2)\in D$. For all $p\in P$ such that
$(p, p_1)\in D$,
\begin{equation}\label{eqn:rl}
(p, p_1\circ p_2)\in D.\tag{$Rl_{pg}$}
\end{equation}
\item[\emph{Right representativity:}] 
Let $p_1,p_2 \in P$ such that $(p_1, p_2)\in D$.
For all $p\in P$ such that  $(p_2, p)\in D$,
\begin{equation}\label{eqn:rr}
(p_1\circ p_2, p)\in D.\tag{$Rr_{pg}$}
\end{equation}
\item[\emph{Representativity:}] 
\begin{equation}\label{eqn:r}
\mbox{$Rl_{pg}$ and $Rr_{pg}$.}\tag{$R_{pg}$}
\end{equation}
The element $p_3$, obtained from composing $p_1$ and $p_2$,
represents the original $p_1$ and $p_2$:
the element $p$ that is composable with $p_1$ will also be composable with $p_3$ from the left side
and the element $p$ that is composable with $p_2$ will also be composable with $p_3$ from the right side.
\end{description}

To define a partial semigroup $(T,D, \circ)$, one must specify one version of associativity properties,
held by the partial binary operation $\circ$, to guarantee that rearranging the parentheses in an expression
$x\circ\cdots  y\circ\cdots  z$ will produce the same result, for all $x,\cdots,y,\cdots,z\in T$. 
It is surprised that there are many ways to ramify associativity properties.
In \cite{LE}, weak associativity, Catenary associativity, intermediate associativity, strong associativity,
strong Catenary associativity, weak Catenary associativity, and $n$-versions for several of these properties
were discussed. 
In this paper, we are interested in the following three associative properties 
$\eqref{eqn:a}$, $\eqref{eqn:cass}$, and $\eqref{eqn:sass}$ on a given partial groupoid
$(P,D,\circ)$.
\begin{description}
\item[\emph{Associativity in \cite{BGSWW}:}]
For all $p_1 , p_2 , p_3\in P$, if $(p_1,p_2)\in D$, $(p_2,p_3)\in D$,
$(p_1,p_2\circ p_3)\in D$, and $(p_1\circ p_2, p_3)\in D$ then
\begin{equation}\label{eqn:a}
p_1\circ (p_2\circ p_3)=(p_1\circ p_2)\circ p_3\neq \,\uparrow.\tag{$A_{pg}$}
\end{equation}
Composing three elements $p_1$, $p_2$, $p_3$ yields the same result regardless of the orders in which $p_1$, $p_2$, $p_3$
are binarily grouped when both $p_1\circ (p_2\circ p_3)$ and $(p_1\circ p_2)\circ p_3$ exist.
\item[\emph{Catenary associativity:}]
For all $p_1, p_2, p_3 \in P$, if $(p_1,p_2)\in D$ and $(p_2,p_3)\in D$ then
\begin{equation}\label{eqn:cass}
p_1\circ (p_2\circ p_3)=(p_1\circ p_2)\circ p_3\neq \,\uparrow.\tag{$CA_{pg}$}
\end{equation}
Composing three elements $p_1$, $p_2$, $p_3$ yields the same result regardless of the orders in which $p_1$, $p_2$, $p_3$
are binarily grouped when both $p_1\circ p_2$ and $p_2\circ p_3$ exist.
\item[\emph{Strong associativity:}]
For all $p_1, p_2, p_3\in P$,
\begin{equation}\label{eqn:sass}
p_1\circ (p_2\circ p_3)=(p_1\circ p_2)\circ p_3\tag{$SA_{pg}$}
\end{equation}
in the sense that both sides of the equation yield the same value or both are simultaneously undefined.
It is equivalent to
either of $(p_1\circ p_2)\circ p_3$ and $p_1\circ (p_2\circ p_3)$ is defined so is the other and they equal.
\end{description}

The relationships between some of the partial groupoid properties above are given in the following two propositions.
\begin{proposition}\label{prop:ca_a_r}
For a given partial groupoid $(P,D,\circ)$,
\begin{enumerate}[$(i)$]
\item
$(P,D,\circ)$ satisfies $\eqref{eqn:wc}$ and $D$ is symmetric
if and only if $(P,D,\circ)$ satisfies $\eqref{eqn:c}$;
\item
$(P,D,\circ)$ satisfies
$\eqref{eqn:cass}$ if and only if  $(P,D,\circ)$ satisfies $\eqref{eqn:a}$ and $\eqref{eqn:r}$.
\end{enumerate}
\end{proposition}
\begin{proof}
\begin{enumerate}[$(i)$]
\item
It is clear.
\item
``$\Rightarrow$":
For all $p_1,p_2,p_3\in P$, if $(p_1,p_2),(p_2,p_3)\in D$, then, by $\eqref{eqn:cass}$, 
$$(p_1\circ p_2)\circ p_3=p_1\circ (p_2\circ p_3)\neq \;\uparrow$$
Hence $(p_1, p_2\circ p_3)\in D$ and $(p_1\circ p_2,p_3)\in D$ and therefore $\eqref{eqn:r}$ holds.
Clearly, $\eqref{eqn:cass}\Rightarrow \eqref{eqn:a}$.
Hence $\eqref{eqn:cass}$ implies $\eqref{eqn:a}$ and $\eqref{eqn:r}$.

``$\Leftarrow$":
Assume that both $p_1\circ p_2$ and $p_2\circ p_3$ exist. Then, by $\eqref{eqn:r}$,  
both $(p_1\circ p_2)\circ p_3$ and $p_1\circ (p_2\circ p_3)$ exist, and so, by  $\eqref{eqn:a}$,
$(p_1\circ p_2)\circ p_3=p_1\circ (p_2\circ p_3)\neq \; \uparrow$. Therefore, $\eqref{eqn:cass}$ holds true.
\end{enumerate}
\end{proof}

\begin{proposition}\label{prop:associativities}
For a partial groupoid $(P,D,\circ)$,
\begin{enumerate}[$(1)$]
\item
strong associativity $\eqref{eqn:sass}$ implies associativity $\eqref{eqn:a}$;
\item
Catenary associativity $\eqref{eqn:cass}$ implies associativity $\eqref{eqn:a}$.
\end{enumerate}
\end{proposition}
\begin{proof}
\begin{enumerate}[$(1)$]
\item
It is clear.
\item
For all $x,y,z\in P$, if $(x\circ y)\circ z \neq \;\uparrow$ and $x\circ (y\circ z)\neq \;\uparrow$,
then $x\circ y\neq \;\uparrow$ and $y\circ z \neq \;\uparrow$.
Hence, by $\eqref{eqn:cass}$, $(x\circ y)\circ z=x\circ (y\circ z)\neq \;\uparrow$,
and therefore \eqref{eqn:a} holds true.
\end{enumerate}
\end{proof}
However, Catenary associativity $\eqref{eqn:cass}$ and strong associativity $\eqref{eqn:sass}$ 
in a partial groupoid are mutually independent as shown by examples in \cite{LE} (\textbf{1}.5.7)
and the converses of statements in Proposition \ref{prop:associativities} do not hold,
shown in Example $\ref{example:properties}.\ref{exam1}$ below.

\begin{example}\label{example:properties}
\begin{enumerate}[$(1)$]
\item\label{exam1}
The partial groupoid $(P_1,\cdot)$, defined in Example $\ref{examp:pargroup}.\ref{xamp:pargroup1}$ above, 
satisfies $\eqref{eqn:i}$ clearly.
It also satisfies the associativity $\eqref{eqn:a}$ since
$$(a\cdot b)\cdot c=b\cdot c=c\mbox{ and }
a\cdot (b\cdot c)=a\cdot c=\;\uparrow.$$
But $(P_1,\cdot)$ does not satisfy strong associativity $\eqref{eqn:sass}$ and Catenary associativity $\eqref{eqn:cass}$.
Hence the converses of the statements in Proposition \ref{prop:associativities} do not hold true.
\item
Since identifying and keeping one copy of the overlaps behaves like the set union,
$({\cal P}ath (G), \sim_p,\between)$ and $({\cal P}ath({\mathcal C}_W), \sim_p, \between)$
in Example $\ref{examp:pargroup}.\ref{xamp:path}$,
$(O/\mathbb{O},\,\mbox{v},\,\sqcup)$ in Example $\ref{examp:pargroup}.\ref{xamp:pargroupMont}$,
and $(E,\smile, \merge)$ in Example $\ref{examp:pargroup}.\ref{xamp:mergeoverlap}$,
satisfy $\eqref{eqn:i}$, $\eqref{eqn:c}$, $\eqref{eqn:a}$, $\eqref{eqn:r}$, $\eqref{eqn:cass}$, and $\eqref{eqn:sass}$.
\item
$(\mathbb{N},\vee)$ and $(\mathbb{N},\wedge)$ in Example $\ref{examp:pargroup}.\ref{xamp:pargroup_naturalnumber}$
satisfy $\eqref{eqn:i}$, $\eqref{eqn:c}$, $\eqref{eqn:a}$, $\eqref{eqn:r}$, $\eqref{eqn:cass}$, and $\eqref{eqn:sass}$.
\item
$(A,*)$ in Example $\ref{examp:pargroup}.\ref{xamp:pargroup3}$ satisfies $\eqref{eqn:a}$
but does not satisfy any of $\eqref{eqn:i}$, $\eqref{eqn:c}$, $\eqref{eqn:r}$, $\eqref{eqn:cass}$, and $\eqref{eqn:sass}$.
\item
$(T,\Join_l)$ in Example $\ref{examp:pargroup}.\ref{xamp:pargroup6}$
satisfies $\eqref{eqn:i}$, $\eqref{eqn:a}$, $\eqref{eqn:r}$, $\eqref{eqn:cass}$, and $\eqref{eqn:sass}$
but does not satisfy $\eqref{eqn:c}$.

$(T,\&)$ in Example $\ref{examp:pargroup}.\ref{xamp:pargroup6}$
satisfies $\eqref{eqn:c}$, $\eqref{eqn:a}$, $\eqref{eqn:r}$, $\eqref{eqn:cass}$, and $\eqref{eqn:sass}$
but does not satisfy $\eqref{eqn:i}$.
\item
$(\mathcal{L},\leadsto,c)$ in Example $\ref{examp:pargroup}.\ref{xamp:pargroup_datalineage}$
satisfies $\eqref{eqn:a}$, $\eqref{eqn:r}$, $\eqref{eqn:cass}$, and $\eqref{eqn:sass}$
but does not satisfy $\eqref{eqn:i}$ and $\eqref{eqn:c}$.
\end{enumerate}
\end{example}

These properties are preserved under partial groupoid homomorphisms shown in the following proposition.
\begin{proposition}\label{prop:preserved}
Given a partial groupoid $(P,D,\circ)$,
the properties: $\eqref{eqn:sym}$, $\eqref{eqn:i}$, $\eqref{eqn:wc}$, $\eqref{eqn:c}$, $\eqref{eqn:r}$, $\eqref{eqn:a}$, 
$\eqref{eqn:cass}$, and $\eqref{eqn:sass}$ are preserved under partial groupoid homomorphisms.
\end{proposition}
\begin{proof}
Let $f:(P,D,\circ)\rw (Q,E,\cdot)$ be a partial groupoid homomorphism.

If $D$ is symmetric, clearly $f(D)=\{(f(x),f(y))|(x,y)\in D\}$ is symmetric too.

If $(P,D,\circ)$ satisfies $\eqref{eqn:i}$,
then, for each $p\in P$, $(p,p)\in P$ and $p\circ p=p$. Hence
for each $f(p)\in f(P)$, $p\in P$,
$$(f(p),f(p))=f(p,p)\in f(D)$$
and
$$f(p)\cdot f(p)=f(p\cdot p)=f(p),$$
and so $\mbox{im}(f)=(f(P),f(D),\cdot)$ satisfies $\eqref{eqn:i}$.

Assume that $(P,D,\circ)$ satisfies $\eqref{eqn:wc}$.
For all $f(p),f(q)\in f(P)$, we have $f(p)\cdot f(q)=f(p\circ q)=f(q\circ p)=f(q)\cdot f(p)$.
Hence $\mbox{im}(f)=(f(P),f(D),\cdot)$ satisfies $\eqref{eqn:wc}$.

If $(P,D,\circ)$ satisfies $\eqref{eqn:a}$,
then, for all $f(p),f(q),f(r)\in \mbox{im}(f)$ such that
$f(p)\cdot(f(q)\cdot f(r))$ and $(f(p)\cdot f(q))\cdot f(r)$ exist, we have
$$f(p)\cdot(f(q)\cdot f(r))=f(p\circ (q\circ r))=f((p\circ q)\circ r)=(f(p)\cdot f(q))\cdot f(r).$$
Hence $\mbox{im}(f)$ satisfies $\eqref{eqn:a}$.

Assume that $(P,D,\circ)$ satisfies $\eqref{eqn:r}$. For all $f(p),f(q),f(r)\in f(P)$
such that 
$$(f(p),f(q)), (f(q),f(r))\in f(D),$$
we have
$(p,q),(q,r)\in D$. Hence $(p\circ q,r)\in D$ and therefore
$(f(p)\circ f(q),f(r))\in f(D)$. Thus, $\mbox{im}(f)$ satisfies $\eqref{eqn:rr}$.
Similarly, $\eqref{eqn:rl}$ is preserved and so is $\eqref{eqn:r}$ under partial groupoid homomorphisms.

If $(P,D,\circ)$ satisfies $\eqref{eqn:sass}$, then, for all $f(p),f(q),f(r)\in f(P)$,
$$f(p)\cdot (f(q)\cdot f(r))=f(p\circ(q\circ r))=f((p\circ q)\circ r)=(f(p)\cdot f(q))\cdot f(r).$$
Hence $\mbox{im}(f)$ satisfies $\eqref{eqn:sass}$.
\end{proof}

\section{Natural Partial Orders in Partial Groupoids}\label{sect:posets}
ER merging aims to combine entity records together to obtain more information.
In order to compare entity record information, a partial order on entity records is needed.

Recall that a {\em partial order} on a set $P$ is a  {\em reflexive, antisymmetric, and transitive} binary relation on $P$.
It is well known that for each semigroup $T$,
the set $E_T$ of all idempotents in $T$ is partially ordered by $\leq_i$:
$$e\leq_i f \mbox{ if and only if }e=ef=fe.$$

A partial order on a semigroup is {\em natural} if it is defined by the operation of the semigroup \cite{M}.

We define $\leq_l$, $\leq_r$, and $\leq$ on a partial groupoid  as follows.
\begin{definition}
Let $(P,D,\circ)$ be a partial groupoid and $p,q\in P$. Define
\begin{enumerate}[$(i)$]
\item
$p\leq_r q$ if and only if $p q=q$;
\item
$p\leq_l q$ if and only if $q p=q$;
\item
$p\leq q$ if and only if $p q=q p=q$.
\end{enumerate}
\end{definition}
Clearly, $\leq$ is the dual of $\leq_i$, when being restricted to the the set of all idempotents of $P$.
The definitions make sense by the following lemma.
\begin{lemma}\label{lemma:partialorders}
Given a partial groupoid $(P,D,\circ)$
\begin{enumerate}[$(i)$]
\item
$\leq$ is antisymmetric;
\item
If $(P,D,\circ)$ satisfies $\eqref{eqn:i}$,
then $\leq_l$, $\leq_r$, and $\leq$ are reflexive;
\item
If $(P,D,\circ)$ satisfies $\eqref{eqn:cass}$,
then $\leq_l$, $\leq_r$, and $\leq$ are transitive.
\end{enumerate}
\end{lemma}
\begin{proof}
\begin{enumerate}[$(i)$]
\item
If $p\leq q$ and $q\leq p$, then
$$p=pq=qp\mbox{ and }q=pq=qp$$ 
and so $p=pq=q$.
Hence $\leq$ is also antisymmetric.
\item
Since $(P,D,\circ)$ satisfies $\eqref{eqn:i}$,
$\leq_l$, $\leq_r$, and $\leq$ are reflexive as $p=p p$ for all $p\in P$.
\item
If $p\leq_l q$ and $q\leq_l r$, then $qp=q$ and $rq=r$. Hence
\begin{eqnarray*}
r   & = & r(qp)\\
   & = & (rq)p\quad(\mbox{by $\eqref{eqn:cass}$})\\
   & = & rp,
\end{eqnarray*}
and therefore $p\leq_l r$. Thus, $\leq_l$ is transitive.
Similarly, $\leq_r$ and $\leq$ are transitive.
\end{enumerate}
\end{proof}

\begin{remark}\label{remark:partialorders}
The condition ``the associativity $\eqref{eqn:cass}$" in Lemma \ref{lemma:partialorders} above
is necessary and cannot be replaced by the associativity $\eqref{eqn:a}$. 

For example, let $(P_1,\cdot)$ be the idempotent partial groupoid defined in 
Example $\ref{examp:pargroup}.\ref{xamp:pargroup1}$.
Then $(P_1,\cdot)$ satisfies $\eqref{eqn:a}$ since
$(a\cdot b)\cdot c=c$ and $a\cdot (b\cdot c)=\;\uparrow$. 
Clearly, $a\leq_r b$ and $b\leq_rc$ but $a \nleq_r c$. 
That is, $\leq_r$ is not transitive.
\end{remark}

Compositions have more information, shown in the following lemma.
\begin{lemma}\label{lemma:comporder}
In a partial groupoid satisfying $\eqref{eqn:i}$ and $\eqref{eqn:cass}$, if $pq$ exists, then
$q\leq_l pq$ and $p\leq_r pq$.
\end{lemma}
\begin{proof}
By $\eqref{eqn:i}$ and $\eqref{eqn:cass}$, we have
$$
(pq)q  =  p(qq) 
= pq
$$
and
$$
p(pq)  = (pp)q
=  pq.
$$
Hence
$q\leq_l pq$ and $p\leq_r pq$.
\end{proof}

Maximal elements are defined by:
\begin{definition}
In a partial groupoid, $m$ is {\em maximal with respect to} $\leq_r \,(\leq_l \mbox{ or } \leq)$ if
$m\leq_r n\,(m\leq_l n\mbox{ or }\,m\leq n)$ implies $n\leq_r m\,(n\leq_l m\mbox{ or } n\leq m)$. 
\end{definition}

The binary relations $\leq_l$, $\leq_r$, and $\leq$ have the same maximal elements shown in the following Lemma \ref{lemma:maximal}
and Proposition \ref{prop:maxelement}.
\begin{lemma}\label{lemma:maximal}
In a partial groupoid satisfying  $\eqref{eqn:i}$,  $\eqref{eqn:sym}$, and $\eqref{eqn:cass}$, the following are equivalent:
\begin{enumerate}[$(i)$]
\item\label{item:lemma_max_1}
$m$ is maximal with respect to $\leq_r$;
\item\label{item:lemma_max_2}
$m$ is maximal with respect to $\leq_l$;
\item\label{item:lemma_max_3}
$m$ is maximal with respect to $\leq$.
\end{enumerate}
\end{lemma}
\begin{proof}
$(\ref{item:lemma_max_1}) \Rightarrow (\ref{item:lemma_max_3})$:
Suppose that $m$ is maximal with respect to $\leq_r$ and that $m\leq n$.
Then $n = mn=nm$ 
and so $m\leq_r n$. Hence $n\leq_r m$ and therefore $n=nm=m$.
Thus, $m$ is maximal with respect to $\leq$.

$(\ref{item:lemma_max_3}) \Rightarrow (\ref{item:lemma_max_1})$:
Assume that $m$ is maximal with respect to $\leq$ and $m\leq_r n$.
Then
$n=mn$ and so $nm$ exists by $\eqref{eqn:sym}$. Thus,
\begin{eqnarray*}
m(nm) & = & (mn)m \quad(\mbox{by  $\eqref{eqn:cass}$})\\
& = & nm\\
& = & (nm)m.\\
\end{eqnarray*}
Hence $m\leq nm$ and therefore $m=nm$ as $m$ is maximal with respect to $\leq$.
Thus, $n\leq_r m$. So $m$ is maximal with respect to $\leq_r$.
Therefore, $(\ref{item:lemma_max_3}) \Rightarrow (\ref{item:lemma_max_1})$.

$(\ref{item:lemma_max_3}) \Leftrightarrow (\ref{item:lemma_max_2})$ is similar to 
$(\ref{item:lemma_max_3}) \Leftrightarrow (\ref{item:lemma_max_1})$.
\end{proof}

Let
$$M_l(P)=\{\mbox{all left maximal elements in partial groupoid }P\},$$
$$M_r(P)=\{\mbox{all right maximal elements in partial groupoid }P\},$$
and
$$M(P)=\{\mbox{all maximal elements in partial groupoid }P\}.$$
By Lemma \ref{lemma:maximal}, we have:
\begin{proposition}\label{prop:maxelement}
In a partial groupoid $(P,D,\circ)$ satisfying $\eqref{eqn:sym}$, $\eqref{eqn:i}$, and $\eqref{eqn:cass}$,
$M_l(P)=M_r(P)=M(P).$
\end{proposition}

Maximal elements form the entity resolution defined in \cite{BGSWW}.
\begin{theorem}\label{theorem:ermax}
Given an ER system $(\mathbb{E}, \approx,\langle\,\rangle)$ satisfying $\eqref{eqn:i}$ and $\eqref{eqn:cass}$
and an instance $I\subseteq \mathbb{E}$ such that $\rst{I}$ is finite,
$ER(I)=M(\rst{I})$.
\end{theorem}
\begin{proof}
For each $r\in \rst{I}$, if $r$ is maximal in $\rst{I}$ with respect to $\leq$, 
then $r\in M(\rst{I})$ and so there exists $r\in M(\rst{I})$ such that $r\leq r$.
Otherwise,  there is $s_1\in \rst{I}$ such that $r\leq s_1$ and $s_1\neq r$.

If $s_1$ is maximal in $\rst{I}$ with respect to $\leq$, 
then $s_1\in M(\rst{I})$ and so there exists $s_1\in M(\rst{I})$ such that $r\leq s_1$.
Otherwise,  there is $s_2\in \rst{I}$ such that $s_1\leq s_2$ and $s_2\neq s_1$.

Continuing in this way, we have the following ascending chain in $\rst{I}$:
$$r< s_1< s_2< \cdots$$
which must terminate as $\rst{I}$ is finite.
Hence there is $s_n\in M(\rst{I})$ such that $r\leq s_n$ and therefore $\rst{I}\leq M(\rst{I})$.
Clearly $M(\rst{I})\subseteq \rst{I}$
and $M(\rst{I})$ is minimal with respect to the properties in the definition of $ER(\rst{I})$.
Thus, $M(\rst{I})=ER(I)$. 
\end{proof}

\begin{remark}
We are interested in not only maximal elements but also minimal or other special elements in some use cases,
e.g., finding atomic elements with minimal information or irreducible generators in a partial groupoid of entity objects.
Dually, we can define minimal elements with respect to $\leq_l$, $\leq_r$, and $\leq$, in a partial groupoid, respectively.
\end{remark}

Let $(P,D,\circ)$ be a partial groupoid and $\preceq$ a given partial order on $P$. 
To study the interactions between $\preceq$ and $\circ$,
we consider the following properties with respect to $\preceq$:

\begin{description}
\item[\emph{Composition gives the least upper bound:}]
For all $p_1, p_2\in P$ such that $(p_1, p_2)\in D$, 
\begin{equation}\label{eqn:lub}
p_1\circ p_2\mbox{  is the least upper bound of }p_1\mbox{ and }p_2\\
\mbox{ with respect to }\preceq.\tag{$LU_{pg}$}
\end{equation}
\item[\emph{$\preceq$ is left compatible with domain and composition}:]
For all $p_1, p_2, p\in P$ such that $p_1\preceq p_2$, 
\begin{equation}\label{eqn:lcomp}
\mbox{if $(p,p_1)\in D$ then $(p,p_2)\in D$ and $p\circ p_1\preceq p\circ p_2$.} \tag{$lCP_{pg}$}
\end{equation}
\item[\emph{$\preceq$ is right compatible with domain and composition}:]
For all $p_1, p_2, p\in P$ such that $p_1\preceq p_2$, 
\begin{equation}\label{eqn:rcomp}
\mbox{if $(p_1,p)\in D$ then $(p_2,p)\in D$ and $p_1\circ p\preceq p_2\circ p$.} \tag{$rCP_{pg}$}
\end{equation}
\item[\emph{$\preceq$ is compatible with domain and composition}:]
If $\preceq$ is 
\begin{equation}\label{eqn:comp}
\mbox{both left and right compatible with domain and composition.}\tag{$CP_{pg}$}
\end{equation}
\end{description}

The following proposition shows that $\eqref{eqn:i}$, $\eqref{eqn:wc}$, $\eqref{eqn:a}$, 
and $\eqref{eqn:r}$ imply $\eqref{eqn:lub}$ and $\eqref{eqn:comp}$.
\begin{proposition}\label{prop:itolu}
Let $(P,D,\circ)$ be a partial groupoid. 
If $(P,D,\circ)$ satisfies $\eqref{eqn:i}$, $\eqref{eqn:wc}$, $\eqref{eqn:a}$, 
and $\eqref{eqn:r}$, then
$(P,D,\circ)$ satisfies $\eqref{eqn:lub}$ and $\eqref{eqn:comp}$ with respect to $\leq$.
\end{proposition}
\begin{proof}
Let $p_1,p_2\in P$ be such that $(p_1,p_2)\in D$. Since 
$$p_1 (p_1 p_2)=(p_1 p_1) p_2=p_1 p_2$$
and
$$(p_1 p_2) p_1=p_1(p_2 p_1)=p_1 (p_1 p_2)=p_1 p_2$$
by $\eqref{eqn:i}$, $\eqref{eqn:a}$, 
and $\eqref{eqn:r}$, we have $p_1\leq p_1 p_2$.
Similarly, we $p_2\leq p_1 p_2$. Then $p_1 p_2$ is an upper bound of $p_1$ and $p_2$.

If $p_1\leq x$ and $p_2\leq x$, then
$$p_1 x=x p_1=x\mbox{ and }p_2 x=x p_2=x.$$
Hence
$$(p_1 p_2) x=p_1 (p_2 x)=p_1 x=x$$
and
$$x (p_1 p_2)=(x p_1) p_2=x p_2=x$$
and therefore $p_1 p_2\leq x$.
Thus, $p_1 p_2$ is the least upper bound of $p_1$ and $p_2$.
So $\eqref{eqn:lub}$ holds true.

If $p_1\leq p_2$ and $(p,p_1)\in D$, then
$$p_2=p_1 p_2=p_2 p_1\mbox{ and }p p_1\mbox{ exists}$$
and so, by $\eqref{eqn:a}$ and $\eqref{eqn:r}$,
$$p p_2=p (p_1 p_2)=(p p_1) p_2\neq\;\uparrow.$$
Thus, $(p,p_2)\in D$.
Compute
\begin{eqnarray*}
(p p_1) (p p_2)   & = & (p p_2) (p p_1)\quad(\mbox{by $\eqref{eqn:c}$})\\
   & = & ((p p_2) p) p_1\quad(\mbox{by $\eqref{eqn:cass}$})\\
   & = & (p (p p_2)) p_1\quad(\mbox{by $\eqref{eqn:cass}$})\\
   & = & p (p_2 p_1)\quad(\mbox{by $\eqref{eqn:i}$ and $\eqref{eqn:cass}$})\\
   & = & p p_2.\\
\end{eqnarray*}
Then $p p_1\leq p p_2$ and so $\eqref{eqn:lcomp}$ holds true.
Similarly, $\eqref{eqn:rcomp}$ is satisfied.
Hence $\eqref{eqn:comp}$ holds.
\end{proof}

On the other hand, when the domain is reflexive, $\eqref{eqn:lub}$ and $\eqref{eqn:comp}$ imply
$\eqref{eqn:i}$, $\eqref{eqn:wc}$, $\eqref{eqn:a}$, and $\eqref{eqn:r}$.
\begin{proposition}\label{prop:lutoi}
Let $(P,D,\circ)$ be a partial groupoid, $D$ reflexive, 
and $\preceq$ a partial order on $P$. 
If $(P,D,\circ)$ satisfies $\eqref{eqn:lub}$ and $\eqref{eqn:comp}$ with respect to $\preceq$, then
$(P,D,\circ)$ satisfies $\eqref{eqn:i}$, $\eqref{eqn:wc}$, $\eqref{eqn:a}$, 
and $\eqref{eqn:r}$ and $\preceq\;=\;\leq$.
\end{proposition}
\begin{proof}
Assume that $(P,D,\circ)$ satisfies $\eqref{eqn:lub}$ and $\eqref{eqn:comp}$ with respect to $\preceq$.
\begin{description}
\item[{\em Idempotence:}]
For all $p\in P$, since $D$ is reflexive, $p p$ exists.
By $\eqref{eqn:lub}$, 
$$p\leq p p\mbox{ and }p p\leq p$$
and so $p p=p$. $\eqref{eqn:i}$ holds true.

\item[{\em Commutativity:}]
For all $p_1,p_2\in P$, 
if $p_1 p_2$ and $p_2 p_1$ exist, then, by $\eqref{eqn:lub}$,
$$p_1 p_2\leq p_2 p_1\mbox{ and }p_2 p_1\leq p_1 p_2$$
and so $p_1 p_2=p_2 p_1$. $\eqref{eqn:wc}$ is satisfied.

\item[{\em Associativity:}]
For all $p_1,p_2,p_3\in P$ such that $p_1 (p_2 p_3)$ and $(p_1 p_2) p_3$ exist,
since $p_1\leq p_1 p_2$ and $(p_1,p_2 p_3)\in D$,
$$p_1 (p_2 p_3)\leq (p_1 p_2) (p_2 p_3).$$
Since $p_1 p_2\leq p_1 (p_2 p_3)$ and
$p_2 p_3\leq p_1 (p_2 p_3)$ by $\eqref{eqn:lub}$, we have
$$(p_1 p_2) (p_2 p_3)\leq p_1 (p_2 p_3).$$
Hence 
$$p_1 (p_2 p_3)=(p_1 p_2) (p_2 p_3).$$
Similarly, $(p_1 p_2) p_3=(p_1 p_2) (p_2 p_3)$.
Therefore,
$p_1 (p_2 p_3)=(p_1 p_2) p_3$, as desired.

$\eqref{eqn:a}$ holds true.
\item[{\em Representativity:}]
If $p_1 p_2$ exists and $(p,p_1)\in D$, then, by $\eqref{eqn:lcomp}$,
$(p, p_1 p_2)\in D$, and so $\eqref{eqn:rl}$ holds.
Similarly, we have $\eqref{eqn:rr}$. Thus,
$\eqref{eqn:r}$ holds.
\end{description}

If $p_1\preceq p_2$, then, by $\eqref{eqn:comp}$ and $\eqref{eqn:i}$,
$$p_1 p_2\preceq p_2 p_2=p_2\mbox{ and }p_2 p_1\preceq p_2 p_2=p_2$$
and so $p_2 p_1=p_1 p_2=p_2$ 
since $p_2\preceq p_1 p_2$ and $p_2\preceq p_2 p_1$ by $\eqref{eqn:lub}$.
Hence $p_1\leq p_2$.

On the other hand,
if $p_1\leq p_2$, then $p_1 p_2=p_2 p_1=p_2$
and so $p_1\preceq p_2$ by $\eqref{eqn:lub}$.
Therefore, $\preceq\; =\;\leq$.
\end{proof}

A partial order on $(P,D, \circ)$, where $D$ is reflexive,  must be the natural partial order 
$\leq$ and $\eqref{eqn:i}$, $\eqref{eqn:wc}$, $\eqref{eqn:a}$, 
and $\eqref{eqn:r}$ hold if and only if
$\eqref{eqn:lub}$ and $\eqref{eqn:comp}$ are satisfied 
by combining Propositions $\ref{prop:itolu}$ and $\ref{prop:lutoi}$.
\begin{theorem}\label{thm:partialordersequal}
Let $(P,D,\circ)$ be a partial groupoid, $D$ reflexive, 
and $\preceq$ a partial order on $P$. 
Then 
$(P,D,\circ)$ satisfies $\eqref{eqn:lub}$ and $\eqref{eqn:comp}$ with respect to $\preceq$ if and only if 
$(P,D,\circ)$ satisfies $\eqref{eqn:i}$, $\eqref{eqn:wc}$, $\eqref{eqn:a}$, 
and $\eqref{eqn:r}$ and $\preceq\;=\;\leq$.
\end{theorem}

\begin{remark}
Sorting and selection were studies in a nonempty set with a transitive relation or a irreflexive and transitive relation \cite{dkmrv,ft}.
Hence we can sort, select, and query the elements in a partial groupoid $(P,D,\circ)$ with a transitive order.
\end{remark}

\section{Reducing Partiality}\label{sect:domaingraph}
It is well known that match and non-match labels for entity resolution is highly imbalanced (class imbalance problem) in general. 
Hence, given an ER system $(\mathbb{E},\approx,\langle\,\rangle)$ and an entity record $e\in \mathbb{E}$,
$e$ can compose only with small part of the set $\mathbb{E}$. That is,
$\langle e,-\rangle$ and  $\langle -,e\rangle$ are defined only on a very small part of the set $\mathbb{E}$.

The partiality $D$ of the partial binary operation $\circ$ of a partial groupoid $(P,D,\circ)$ 
is measured by its domain $D$.
Since $D$ is a subset of $P\times P$,
$D$ is a binary relation on $P$ and can be represented as a directed graph.

\begin{definition}
The {\em domain graph} of a finite partial groupoid $(P,D,\circ)$, denoted by $\mathcal{G}_P$, 
is defined to be the directed graph with
\begin{description}
\item[nodes:]
all elements of $P$,
\item[edges:]
there is an edge from node $p_1$ to node $p_2$ if $p_1\circ p_2$ is defined. 
\end{description}
A partial groupoid $(P,D,\circ)$ is {\em connected} if its domain graph ${\mathcal G}_P$ is a connected graph.
\end{definition}
If $D$ is symmetric, then ${\mathcal G}_P$ can be viewed as a simple undirected graph.

\begin{example}\label{exam:domaingraph}
\begin{enumerate}
\item\label{exam:domaingraph1}
Domain graphs of partial groupoids defined in Examples $\ref{examp:pargroup}.\ref{xamp:pargroup1}$ and 
$\ref{examp:pargroup}.\ref{xamp:pargroup3}$
are the following directed graphs, respectively:
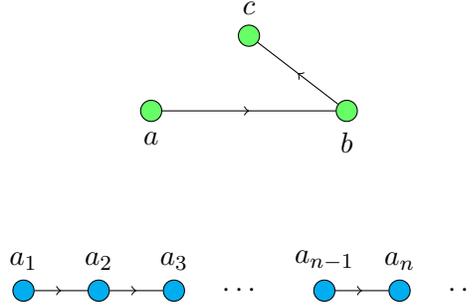
\[\begin{tikzpicture}[x=1.3cm, y=1cm,
    every edge/.style={
        draw,
        postaction={decorate,
                    decoration={markings,mark=at position 0.5 with {\arrow{>}}}
                   }
        }
]
\vertex (a) at (0,0) [fill=green!60, scale =0.8, minimum size = 10pt,label =-90:$a$]{};
\vertex (b) at (2,0) [fill=green!60, scale =0.8, minimum size = 10pt,label =-90:$b$]{};
\vertex (c) at (1,1)[fill=green!60, scale =0.8, minimum size = 10pt,label =90:$c$]{};
\path
(a) edge (b)
(b) edge (c)
;
\end{tikzpicture}\]

\[\begin{tikzpicture}[
    every edge/.style={
        draw,
        postaction={decorate,
                    decoration={markings,mark=at position 0.5 with {\arrow{>}}}
                   }
        }
]
\vertex (a1) at (0,0) [fill=cyan, scale =0.8, minimum size = 10pt,label =90:$a_1$]{};
\vertex (a2) at (1,0) [fill=cyan, scale =0.8, minimum size = 10pt,label =90:$a_2$]{};
\vertex (a3) at (2,0)[fill=cyan, scale =0.8, minimum size = 10pt,label =90:$a_3$]{};
\vertex (a) at (2.5,0)[minimum size = 0pt,label =0:$\cdots$]{};
\vertex (a4) at (4,0)[fill=cyan, scale =0.8, minimum size = 10pt,label =90:$a_{n-1}$]{};
\vertex (a5) at (5,0)[fill=cyan, scale =0.8, minimum size = 10pt,label =90:$a_n$]{};
\vertex (a) at (5.5,0)[minimum size = 0pt,label =0:$\cdots$]{};
\path
(a1) edge (a2)
(a2) edge (a3)
(a4) edge (a5)
;
\end{tikzpicture}\]
\end{enumerate}
\end{example}

Recall that a {\em path} in an undirected graph is a sequence of edges which links a sequence of distinct vertices.
A {\em directed path} in a directed graph is a sequence of edges in the same direction, which links a sequence of distinct
vertices.
A pair of vertices $u$ and $v$ in an undirected graph 
is {\em connected} if there is a path between the vertices. 
A pair of vertices $u$ and $v$ in a directed graph 
is {\em strongly connected} if there is a directed path from $u$ to $v$ and a directed path from $v$ to $u$. 

An undirected (A directed) graph is {\em complete} if each pair of distinct vertices is connected
by a unique edge (a pair of unique edges, one in each direction).

\begin{proposition}
Let $(P,D,\circ)$ be a finite partial groupoid and $D$ symmetric. Then
$(P,D,\circ)$ is total if and only if its domain graph ${\mathcal G}_P$ is a complete graph.
\end{proposition}
\begin{proof}
A partial semigroup $(P,D,\circ)$ is total if and only if 
for all $p_1,p_2\in P$, $p_1\circ p_2$ is defined if and only if
each pair of distinct nodes in ${\mathcal G}_P$ is connected by its unique edge  if and only if
its domain graph ${\mathcal G}_P$ is complete.
\end{proof}

A {\em connected component} of an undirected graph is a connected subgraph that is not contained in any other connected
subgraph properly.
As it is well known, each finite undirected graph can be partitioned into disjoint connected components in linear time.
Let $\{C_1,\cdots,C_n\}$ be all connected components of ${\mathcal G}_P$.
Then 
$P=\bigcup_{i=1}^{n}P_i,$
where each $P_i=P|_{C_i}$ is the subgroupoid of $P$ restricted to $C_i$, $P_i\cap P_j=\emptyset$,
and for each $s\in P_i, t\in P_j$, both $s\circ t$ and $t\circ s$ are undefined if $i\neq j$.
Each $P_i$ is called a {\em connected component partial subgroupoid} of $P$.

Total groupoids (semigroups) are studied more extensively than partial ones. Hence it is desirable that
a partial groupoid (semigroup) can be covered by some total groupoids (semigroups).

A {\em clique} of a graph $G$ is a complete subgraph of $G$ and a {\em maximal clique} of $G$ is the one not being contained
in any other properly.
A {\em clique cover} of $G$ is a set of cliques of $G$, whose union is $G$.

There are the designed algorithms to find maximal clique and a clique cover for a given undirected (directed) graph
\cite{pull,sch,wu} so that one can decompose a partial groupoid $P$ into a union of total subgroupoids that correspond to 
the clique cover of ${\mathcal G}_P$.

Another way to remove the partiality of a partial groupoid (semigroup) $P$ is to extend $P$ to a total groupoid (semigroup) 
$T$ such that $P$ can be part of $T$.
\begin{proposition}\label{lemma:nullext}
Let $T^*=T\cup \{\uparrow\}$ be the null extension of  a partial semigroup $T$:
for all $s\in T$, $s\circ \uparrow\;=\;\uparrow\circ s=\;\uparrow$ and $s\circ t=\;\uparrow$ if $s\circ t$ is undefined.
Then 
$T^*$ is a (an idempotent)  semigroup if and only if $T$ is a (an idempotent) partial groupoid
satisfying $\eqref{eqn:sass}$.
\end{proposition}
\begin{proof}
``$\Leftarrow$":
For all $x,y,z \in T^*$, if one of $x,y,z$ is $\uparrow$, 
then $(xy)z=\;\uparrow\;=x(yz)$.

For all $x,y,x\in T^*$, if both $(xy)z$ and $x(yz)$ are defined, then
$(xy)z=x(yz)$.
If both $(xy)z$ and $x(yz)$ are not defined, then
$(xy)z=\;\uparrow\;=x(yz)$. Hence $T^*$ is a semigroup.

``$\Rightarrow$":
For all $x,y,z\in T$, since $T^*$ is a semigroup, we have
$x(yz)=(xy)z$ in $T^*$. Hence $x(yz)=(xy)z$ in $T$ and therefore $T$ satisfies $\eqref{eqn:sass}$.

Clearly, $T^*$ is idempotent if and only if $T$ is idempotent.
\end{proof}

\section{Removing Noncommutativity}\label{sect:commutativegroupoid}
Commutativity $\eqref{eqn:ec}$ is important in \cite{BGSWW}. In this section, we shall show that each partial groupoid,
satisfying $\eqref{eqn:nr}$ and $\eqref{eqn:sym}$, has a homomorphic image that is a commutative partial groupoid.
We consider the following condition on a partial groupoid $(P,D,\circ)$:
\begin{equation}\label{eqn:nr}
r_1\cdots r_kr_1\cdots r_k=r_1\cdots r_k\tag{$NR_{pg}$}
\end{equation}
for all $r_1,\cdots,r_k\in P$, where $k\geq 1$ is an integer

Given an idempotent semigroup $T$, the associativity implies that rearranging the parentheses 
in an expression $r_1\cdots r_kr_1\cdots r_k$
will have the same result for all $r_1,\cdots,r_k\in T$. Hence 
$$r_1\cdots r_kr_1\cdots r_k=(r_1\cdots r_k)(r_1\cdots r_k)=r_1\cdots r_k$$
and therefore each idempotent semigroup satisfies $\eqref{eqn:nr}$ and $\eqref{eqn:sym}$.

Clearly, if a partial groupoid $(P,D,\circ)$ satisfies $\eqref{eqn:nr}$ then it satisfies $\eqref{eqn:i}$.

$\eqref{eqn:nr}$ is preserved under partial groupoid homomorphisms:
\begin{lemma}\label{lemma:nrinv}
Let $f: (P,D,\circ)\rw (Q,E,\cdot)$ be a partial
groupoid homomorphism. Then
\begin{enumerate}[$(i)$]
\item
If $(P,D,\circ)$ satisfies $\eqref{eqn:nr}$, then so does $\mbox{\rm im}(f)$;
\item
If $(P,D,\circ)$ is finitely generated, so is $\mbox{\rm im}(f)$.
\end{enumerate}
\end{lemma}
\begin{proof}
\begin{enumerate}[$(i)$]
\item
It is clear.
\item
If $(P,D,\circ)$ is finitely generated partial groupoid by $\{p_1,\cdots,p_n\}$, then
$\mbox{im}(f)=(f(P).f(D),\cdot)$ is finitely generated partial groupoid by $\{f(p_1),\cdots,f(p_n)\}$.
\end{enumerate}
\end{proof}

\begin{definition}
Define the relation $\sim_c$ on $P$ by
$$p\sim_c q\mbox{ if and only if } pqp=p \mbox{ and } qpq=q.
$$
\end{definition}
Obviously, if $p\sim_c q$ then both $pq$ and $qp$ exist.

\begin{lemma}\label{lemma:rep}
In a given partial groupoid $(P,D,\circ)$,
if 
$$x_1\cdots x_i=y_1\cdots y_j\mbox{ and }p=p_1\cdots p_sx_1\cdots x_iq_1\cdots q_t,$$ 
then
$$p=p_1\cdots p_sx_1\cdots x_iq_1\cdots q_t=p_1\cdots  p_sy_1\cdots  y_jq_1\cdots q_t$$
for all $x_1,\cdots,x_i,y_1\cdots,y_j,p,p_1,\cdots,p_s,q_1,\cdots,q_t,r_1,\cdots,r_k\in P$,
where $s\geq 0$, $t\geq 0$, $i\geq 1$, and $j\geq 1$ are integers.
\end{lemma}
\begin{proof}
The proof goes by induction on $j$.

If $j=1$, then $x_1\cdots x_i=y_1\mbox{ and }p=p_1\cdots p_sx_1\cdots x_iq_1\cdots q_t$.
Since both 
$$x_1\cdots x_i\mbox{ and }p_1\cdots p_sx_1\cdots x_iq_1\cdots q_t$$ 
result in a single element,
the operation $\circ$ can be performed regardless of how the elements in the expressions 
are grouped or where the parentheses are put.
Hence 
$$p=p_1\cdots p_sx_1\cdots x_iq_1\cdots q_t=p_1\cdots  p_sy_1q_1\cdots q_t.$$

Suppose that the statement in the lemma holds if $j=h$. For the case of $h+1$, assume that
$$x_1\cdots x_i=y_1\cdots y_{h+1}\mbox{ and }p=p_1\cdots p_sx_1\cdots x_iq_1\cdots q_t.$$ 
Then 
$$x_1\cdots x_i=y_1\cdots y_{h-1}(y_h y_{h+1})\mbox{ and }p=p_1\cdots p_sx_1\cdots x_iq_1\cdots q_t.$$ 
By the induction hypothesis,
\begin{eqnarray*}
p & = & p_1\cdots p_sx_1\cdots x_iq_1\cdots q_t\\
& = & p_1\cdots  p_sy_1\cdots  y_{h-1}(y_hy_{h+1})q_1\cdots q_t\quad(\mbox{by induction hypothesis})\\
& = & p_1\cdots  p_sy_1\cdots  y_hy_{h+1}q_1\cdots q_t .\quad(\mbox{since the expression results in a single element})\\
\end{eqnarray*}
Hence the lemma holds for $j=h+1$, as desired.
\end{proof}

For $p,q\in P$, $p$ is a {\em left-divisible} ({\em right-divisible}) by $q$ or $q$ is a {\em left-factor} ({\em right-factor}) of $p$, 
denoted by $q|_lp$ ($q|_rp$), if there is $x\in P$ such that $p=qx$ ($p=xq$).

\begin{lemma}\label{lemma:comp}
Given a partial groupoid $(P,D,\circ)$ satisfying $\eqref{eqn:nr}$ and $p,q,r\in P$, when the operations related exist,
\begin{enumerate}[$(i)$]
\item\label{lemma:comp1}
$pq\sim_c qp$;
\item\label{lemma:comp11}
If $p\sim_c q$ and $p\sim_c r$, then $pqrp=prqp=p$;
\item\label{lemma:comp12}
$\sim_c$ is an equivalence relation;
\item\label{lemma:comp2}
If $p|_lq$ and $q|_lp$, then $pr\sim_c qr$ and $rp\sim_c rq$;
\item\label{lemma:comp3}
If $p|_rq$ and $q|_rp$, then $pr\sim_c qr$ and $rp\sim_c rq$;
\item\label{lemma:comp4}
If $p\sim_c q$, then $pr\sim_c qr$ and $rp\sim_c rq$.
\end{enumerate}
\end{lemma}
\begin{proof}
\begin{enumerate}[$(i)$]
\item
Since $pq=pq\;pq=pq\;qp\;pq$ and $qp=qp\;qp=qp\;pq\;qp$ by \eqref{eqn:nr}, we have $pq\sim_c qp$.
\item
Since $p\sim_c q$ and $p\sim_c r$, by Lemma \ref{lemma:rep} we have
\begin{eqnarray*}
p & = & pqp\quad(\mbox{since }p\sim_c q)\\
& = & prp\;qp\quad(\mbox{since }prp=p)\\\
& = & p\;rpqrpq\;p\quad(\mbox{since }rpq = rpqrpq)\\
& = & pqrp\quad(\mbox{since }p=prp=pqp)
\end{eqnarray*}
and
\begin{eqnarray*}
p & = & pqp\;prp\quad(\mbox{since }p=pp)\\
& = & p\;qpr\;p\quad(\mbox{since }qppr=qpr)\\
& = & p\;qprqpr\;p\quad(\mbox{since }qprqpr=qpr)\\
& = & pqp\;rq\; prp\\
& = & prqp.\\
\end{eqnarray*}
\item
For each $p\in P$, by $\eqref{eqn:nr}$, $ppp=p$ and so $p\sim_c p$. 
Hence $\sim_c$ is reflexive.

If $p\sim_c q$, then 
$pqp=p\mbox{ and }qpq=q$
and so $qpq=q\mbox{ and }pqp=p$. Thus, $q\sim_c p$. Therefore, $\sim_c$ is symmetric.

Assume that $p\sim_c q$ and $q\sim_c r$. Then, by $(\ref{lemma:comp11})$,
$qprq=qrpq=q$
and so, by Lemma \ref{lemma:rep},
\begin{eqnarray*}
p & = & pqqp\quad(\mbox{since }p=pqp\mbox{ and }q=qq)\\
& = & pqp \;rqqr\;pqp\quad(\mbox{since }qprq=qrpq=q)\\
& = & p\;rqr\;p\quad(\mbox{since }p=pqp,q=qq)\\
& = & prp\quad(\mbox{since }r=rqr)\\
\end{eqnarray*}
and
\begin{eqnarray*}
r & = & rqr\quad(\mbox{since }r=rqr)\\
& = & r\;qprq\;r\quad(\mbox{since }q=qprq)\\
& = & rqp\;r \quad(\mbox{since }rqr=r)\\
& = & r\;qrpq\;pr \quad(\mbox{since }q=qrpq)\\
& = & rpr.\quad(\mbox{since }rqr=r\mbox{ and }pqp=p)\\
\end{eqnarray*}
Hence $p\sim_c r$ and therefore $\sim_c$ is transitive.
\item
Suppose that $p=qx$ and $q=py$ for some $x,y\in P$. Then
\begin{eqnarray*}
pr \;qr \;pr & = & prqr\;qx\;r\\
& = & prqxr\quad(\mbox{since }rqrq=rq)\\
& = & pr
\end{eqnarray*}
and
\begin{eqnarray*}
qr\;pr\;qr & = & qrpr\;py\;r\\
& = & qrpyr \quad(\mbox{since }rprp=rp)\\
& = & qr.\\
\end{eqnarray*}
So $pr\sim_c qr$. Similarly,
$$
rp \;rq \;rp  =  rprq\;r\; qx
=  rprqx
= rp
$$
and
$$
rq\;rp\;rq  = rqrp\;r\;py
= rqrpy 
= rq.
$$
Hence $rp\sim_c rq$.
\item
Similar to $(\ref{lemma:comp2})$.
\item
If $p\sim_c q$, then $p|_lpq$, $pq|_lq$ and $pq|_rq$, $q|_rpq$.
Hence, by $(\ref{lemma:comp2})$ and $(\ref{lemma:comp3})$, $pr\sim_c pqr\sim_c qr$.
By $(\ref{lemma:comp1})$, we have $rp\sim_c rq$.
\end{enumerate}
\end{proof}

$\sim_c$ is compatible with composition, shown in:
\begin{lemma}\label{lemma:simcomp}
Given a partial groupoid $(P,D,\circ)$ satisfying $\eqref{eqn:sym}$ and $\eqref{eqn:r}$,
if  $p_1\sim_c q_1$ and $p_2\sim_c q_2$, then
$(p_1,p_2)\in D$ if and only if $(q_1,q_2)\in D$.
\end{lemma}
\begin{proof}
Since $p_i\sim_c q_i$, $i=1,2$, we have $p_i=p_iq_ip_i$ and $q_i=q_ip_iq_i$, $i=1,2$.

If $(p_1,p_2)\in D$, then, by $\eqref{eqn:sym}$ and $\eqref{eqn:r}$, $(p_1,q_2p_2q_2)\in D$. That is, $(p_1,q_2)\in D$.
Similarly, $(q_1p_1q_1,q_2)\in D$ by $\eqref{eqn:sym}$ and $\eqref{eqn:r}$ and so $(q_1,q_2)\in D$.

Since $\sim_c$ is an equivalence relation, if $(q_1,q_2)\in D$, then, symmetrically, $(p_1,p_2)\in D$.
Hence $(p_1,p_2)\in D$ if and only if $(q_1,q_2)\in D$.
\end{proof}

Now, $\sim_c$ is congruence by the following lemma.
\begin{lemma}\label{lemma:congru}
In a given a partial groupoid $(P,D,\circ)$ satisfying $\eqref{eqn:nr}$,
$\sim_c$ is congruence: $p\sim_c p'$ and $q\sim_c q'$ imply $pq\sim_c p'q'$.
\end{lemma}
\begin{proof}
By Lemmas $\ref{lemma:rep}$ and $\ref{lemma:comp}\;(\ref{lemma:comp4})$, 
$pq\sim_c p'q\sim_c p'q'$.
\end{proof}

Given a partial groupoid $(P,D,\circ)$, by Lemma \ref{lemma:congru},
we have a quotient partial groupoid $(P/\!\!\sim_c,D/\!\!\sim_c,\star)$, denoted by $\mathcal{Q}(P,D,\circ)$, 
where $\rst{x}$ is the equivalence class of $x$, with respect to $\sim_c$,
$\rst{x}\star \rst{y}=\rst{x\circ y}$,
$D/\!\!\sim_c=\{(\rst{x},\rst{y})\,|\,(x,y)\in D\}$,
and a partial groupoid homomorphism $(P,D,\circ)\rw \mathcal{Q}(P,D,\circ)$, sending $x$ to $\rst{x}$.
By Lemmas $\ref{lemma:comp}$, $\ref{lemma:simcomp}$, and $\ref{lemma:congru}$, we have the following proposition:

\begin{proposition}\label{prop:tocomm}
Let $(P,D,\circ)$ be a partial groupoid satisfying $\eqref{eqn:sym}$ and $\eqref{eqn:nr}$. 
Then there are a commutative partial groupoid 
$(Q,E,\cdot)$ and a  surjective partial groupoid homomorphism
$$f: (P,D,\circ)\rw (Q,E,\cdot).$$
\end{proposition}

Applying Proposition \ref{prop:tocomm} to a finitely generated partial groupoid, we have:
\begin{corollary}
Let $P=[ p_1,\cdots,p_n]$ be a finitely generated partial groupoid satisfying $\eqref{eqn:sym}$ and $\eqref{eqn:nr}$. 
Then there are a commutative finitely generated partial groupoid $Q$ and a  surjective partial
groupoid homomorphism $\phi: P\rw Q$.
\end{corollary}
\begin{proof}
By Proposition \ref{prop:tocomm}, there are a commutative partial groupoid $Q$ and a  surjective partial
groupoid homomorphism $\phi: P\rw Q$. Since $\phi$ is surjective, $Q=[ \phi(p_1),\cdots,\phi(p_n)]$
is finitely generated.
\end{proof}

Each equivalence class with $\sim_c$ is a semigroup shown in the following proposition.
\begin{proposition}\label{prop:tosemi}
Let $(P,D,\circ)$ be a partial groupoid satisfying $\eqref{eqn:nr}$ and $\rst{p}\in P/\!\!\sim_c$. Then
\begin{enumerate}[$(i)$]
\item\label{prop:tosemi_i}
$x_1\cdots x_k=x_1x_k$ for all $x_1,\cdots,x_k\in \rst{p}$ and $k\geq 2$ is an integer; 
\item\label{prop:tosemi_ii}
$\rst{p}$ is a semigroup with respect to $\circ$;
\item\label{prop:tosemi_iii}
$P$ is a disjoint union of some semigroups: $P=\cup_{x\in P}\rst{x}$,
where $\rst{x}$ is the equivalence class of $x$ with respect to $\sim_c$.
\end{enumerate}
\end{proposition}
\begin{proof}
\begin{enumerate}[$(i)$]
\item
It is clear when $k=2$. If $k=3$, then
\begin{eqnarray*}
x_1x_2x_3& = & x_1x_3x_1\;x_2x_3\\
& = & x_1x_3\quad(x_3x_1x_2x_3 = x_3\mbox{ by Lemma }\ref{lemma:comp} (\ref{lemma:comp11})).\\
\end{eqnarray*}
Assume that $x_1\cdots x_i=x_1x_i$ for any integer $i$ such that $2\leq i\leq k-1$. Then
\begin{eqnarray*}
x_1\cdots x_k& = & x_1\cdots x_{k-1}x_k\\
& = & x_1x_{k-1}x_k\quad(\mbox{By induction hypothesis})\\
& = & x_1x_k,
\end{eqnarray*}
as desired.
\item
For all $x,y\in \rst{p}$, $p\sim_c x$ and $p\sim_c y$. By Lemma $\ref{lemma:comp} (\ref{lemma:comp11})$,
$pxyp=p$. By $(\ref{prop:tosemi_i})$, $xypxy=xy$. Then $xy\sim_c p$ and so $xy\in \rst{p}$.

By $(\ref{prop:tosemi_i})$, $xyz=xz$ is a single value for all $x,y,z\in \rst{p}$. 
Hence the binary operation of $P$ within $\rst{p}$
is associative
and therefore $\rst{p}$ is a semigroup with respect to the binary operation of $P$.
\item
Since $\sim_c$ is an equivalence relation on $P$, $P=\cup_{x\in P}\rst{x}$ with disjoint unions.
By $(\ref{prop:tosemi_ii})$, each $\rst{x}$ is a semigroup.
\end{enumerate}
\end{proof}

The quotient operation $\mathcal{Q}$ is idempotent.
\begin{proposition}\label{prop:QQ=Q}
$\mathcal{Q}^2(P,D,\circ) =\mathcal{Q}(P,D,\circ)$ for all partial groupoid $(P,D,\circ)$ satisfying $\eqref{eqn:nr}$.
\end{proposition}
\begin{proof}
If $(P,D,\circ)$ satisfies $\eqref{eqn:nr}$,
then clearly $\mathcal{Q}(P)$ satisfies $\eqref{eqn:nr}$ too.  $\mathcal{Q}^2(P,D,\circ) =\mathcal{Q}(P,D,\circ)$ since
$$
\infer[\Leftrightarrow]
{
    \infer[\Leftrightarrow]
    {x\sim_c y \;\;in\;\;P}
    {\infer[\Leftrightarrow]
         {xyx=xyx, xyx=x, yxy=yxy, yxy=y\;\;in\;\; P}
         {xyx\;x\;xyx=xyx,x\;xyx\;x=x, yxy\;y\;yxy=yxy, y\;yxy\;y =y\;\;in\;\;P}
    }
 }
{  \infer[\Leftrightarrow]
    {
         \infer[\Leftrightarrow]
         {xyx\sim_c x \;\;and\;\; yxy\sim_c y \;\;in\;\;P}
         {\rst{x}\;\rst{y}\;\rst{x}=\rst{x}\;\;and\;\;\rst{y}\;\rst{x}\;\rst{y}=\rst{y}\;\;in\;\;\mathcal{Q}(P)}
    }
    {
    \rst{x}\sim_c \rst{y}\;\;in\;\;\mathcal{Q}(P)
    }
}
$$
\end{proof}

\section{Entity Resolution on a Finitely Generated Partial Groupoid}\label{chapter:erongroupoids}
Recall that each ER system $(\mathbb{E},\approx,\langle\;\rangle)$ gives rise to a partial groupoid 
$(\mathbb{E},D, \langle\;\rangle)$,
where $D$ is the domain of $\langle\,\rangle$, specified by $\approx$.
Given an instance $I\subseteq \mathbb{E}$,
in \cite{BGSWW}, Benjelloun et al.  defined the entity resolution $ER(I)$ as
the smallest subset $I'\subseteq \rst{I}$ such that $\rst{I}$ is dominated by $I'$. 
$ER(I)$ exists and is unique (Proposition 2.1 \cite{BGSWW})
and can be computed efficiently if $\eqref{eqn:ei}$, $\eqref{eqn:ec}$, $\eqref{eqn:ea}$, $\eqref{eqn:er}$ are satisfied
(Propositions 3.2, 3.3, and 4.1\cite{BGSWW}).
Clearly, $\rst{I}$ is the partial subgroupoid $[I]$, finitely generated by $I$,
in $(\mathbb{E},D, \langle\;\rangle)$.

In Section \ref{sect:properties}, we discussed the properties 
$\eqref{eqn:i}$, $\eqref{eqn:wc}$, $\eqref{eqn:c}$, $\eqref{eqn:r}$, $\eqref{eqn:a}$, $\eqref{eqn:cass}$, 
and $\eqref{eqn:sass}$, which are preserved under partial groupoid homomorphisms.
In Section \ref{sect:posets}, we studied the natural partial oder $\leq$ on a given partial groupoid.
In this section, we shall show that a finitely generated partial groupoid provides another generic setting for entity resolution.

Some binary operations are not associative or commutative necessarily, 
e.g., infix operation ``${\tt :-}$" in Prolog, division, and average are not associative while
left join, function composition, and division are not commutative.
We shall start 
with an arbitrary finitely generated partial groupoid $(P,D,\circ)$ without assuming any properties.
Then we shall study the setting with the properties that imply the existence of a natural partial order and the finiteness
of the setting. 
Throughout this section, $(P,D,\circ)$ is a finitely generated partial groupoid by $p_1,\cdots,p_n$, where
$$P=[p_1,\cdots,p_n]=\bigcup^{+\infty}_{i=1}\{p_1,\cdots,p_n\}^i$$ 
Without loss of generality, we assume that the generating set $\{p_1,\cdots,p_n\}$ is irreducible by Theorem \ref{theorem:irreduciblegenset} and $P$ is connected.
Each $p\in P$ can be written as
$$p=x_1\cdots x_i\cdots x_l$$
with ``(" and ``)" being inserted binarily, where $x_i\in\{p_1,\cdots,p_l\}$, $i=1,\cdots l$.

\begin{definition}\label{def:full}
An element $p$ in a partial groupoid $P$ is called {\em left full (right full)}
if for all $x\in P$, the existence of $xp$ implies $xp=p$. $\big(\mbox{the existence of the }px$ implies $px=p.\big)$
$p$ is called {\em full} if it is both left and right full.
\end{definition}

Let
$$F_l(P,D,\circ)=\{p\in P\,| \,p \mbox{ is left full}\},$$
$$F_r(P,D,\circ)=\{p\in P\,| \,p \mbox{ is right full}\},$$
and
$$F(P,D,\circ)=\{p\in P\,| \,p \mbox{ is full}\}.$$
Clearly, 
$$F(P,D,\circ)=F_l(P,D,\circ)\cap F_r(P,D,\circ).$$

For a finitely generated partial groupoid $(P,D,\circ)$
without any other conditions being required, we define
\begin{equation}\label{eqn:er0}
ER_f(P,D,\circ)=\{\mbox{full elements in }P\}.\tag{$ER_f$}
\end{equation}

For a finitely generated partial groupoid $(P,D,\circ)$ satisfying $\eqref{eqn:i}$ and $\eqref{eqn:cass}$, 
$P$ is partially ordered by Lemma $\ref{lemma:partialorders}$. We define
\begin{equation}\label{eqn:er1}
ER_m(P,D,\circ)=\{\mbox{maximal elements in }P\}.\tag{$ER_m$}
\end{equation}
By Theorem \ref{theorem:ermax}, for a given partial groupoid $(P,D,\circ)$ and an instance $I$ such that
$\rst{I}$ is finite and satisfies $\eqref{eqn:i}$ and $\eqref{eqn:cass}$, 
$$ER(I)=M(\rst{I})=ER_m(\rst{I}),$$
where $ER(I)$ is defined in \cite{BGSWW} (Definition 2.3).

Full elements are the same as the maximal elements as shown  in the following proposition. 
\begin{proposition}\label{prop:max=full}
For a partial groupoid $(P,D,\circ)$ satisfying $\eqref{eqn:i}$ and $\eqref{eqn:cass}$,
$$F_l(P,D,\circ)=M_l(P,D,\circ), F_r(P,D,\circ)=M_r(P,D,\circ), \mbox{ and }F(P,D,\circ)=M(P,D,\circ).$$
\end{proposition}
\begin{proof}
For each $p\in F_l(P,D,\circ)$, if $p\leq_lx$, then $xp=x$.
Since $p$ is left full and $xp$ exists, $p=xp=x$. Hence $p=px$ and therefore $x\leq_lp$. Thus, $p\in M_l(P,D,\circ)$.
So $F_l(P,D,\circ)\subseteq M_l(P,D,\circ)$.

Conversely, for all $p\in M_l(P,D,\circ)$, assume that $xp$ exists. Since $p\leq_l xp$ by Lemma \ref{lemma:comporder} and $p$ is left maximal, 
$xp\leq_l p$.
Hence $p=p(xp)=(px)p=xp$ and therefore $p$ is left full. Thus, $M_l(P,D,\circ)\subseteq F_l(P,D,\circ)$.
Therefore $F_l(P,D,\circ)=M_l(P,D,\circ)$, as desired.

Similarly, we have $F_r(P,D,\circ)=M_r(P,D,\circ)$.

Finally, by Proposition \ref{prop:maxelement},
$M_l(P,D,\circ)=M_r(P,D,\circ)=M(P,D,\circ)$. Hence 
$$M(P,D,\circ)=M_l(P,D,\circ)\cap M_r(P,D,\circ)=F_l(P,D,\circ)\cap F_r(P,D,\circ)=F(P,D,\circ).$$
\end{proof}

Given an ER system $(\mathbb{E},\approx,\langle\,\rangle)$ and a finite instance $I\subseteq \mathbb{E}$,
infinitely many elements (in the merge closure $\rst{I}$) may be generated from $I$ using $\approx$ and $\langle\,\rangle$ \cite{BGSWW}.
To compute entity resolution of $I$ efficiently, Benjelloun et al. \cite{BGSWW} 
introduced four important properties $\eqref{eqn:ei}$, $\eqref{eqn:ec}$, $\eqref{eqn:ea}$, and $\eqref{eqn:er}$
to guarantee the finiteness of $\rst{I}$ and the existence of the natural partial order in $\rst{I}$.

The merge closure $\overline{I}$, containing all possible merges from $I$, is $[I]$,
the partial subgroupoid generated by $I$ in the partial groupoid $(\mathbb{E}, \approx, \langle\,\rangle)$. Hence
$\overline{I}=[I]=\bigcup_{i=1}^{+\infty}I^i.$
In general, the chain of sets $I, I^2, \cdots, I^n,\cdots$ may not terminate nor change periodically
so that $[I]$ is infinite as shown in the following examples.

\begin{example}\label{example:parsemigroup}
\begin{enumerate}
\item
Let $U=\{a_1,\cdots,a_n,\cdots\}$ be an infinite set of objects. Define
$$
a_ia_i  = a_i\mbox{ and }
a_ia_{i+1} =  a_{i+2},$$
for all positive integers $i$
and no other compositions are defined. Then $U$ is a partial groupoid.
Let $I=\{a_1,a_2\}$. Then
\begin{eqnarray*}
I^2 &  = & \{a_1,a_2,a_3\}, \\
I^3 & = & \{a_1,a_2,a_3,a_4\}, \\
& \vdots & \\
I^n & = & \{a_1,\cdots,a_{n+1} \}, \\
& \vdots & \\ 
\end{eqnarray*}
Hence 
$$\overline{I}=[I]=\bigcup_{i=1}^{+\infty}I^i=\Big\{a_1,a_2,\cdots, a_n,\cdots\Big\}=U,$$
which is not finite.
\item
Let $T$ be the set of tables in a rational database
and $\&$ the row bind defined in Example \ref{examp:pargroup}.\ref{xamp:pargroup6}.
Then $(T,\&)$ is a partial groupoid.
Let $t$ be a nonempty table in $T$
and 
$$t^n=\overbrace{t\;\&\;\cdots \;\&\; t}^{n\; times}.$$
Hence
$$[t]=\{t,t^2,\cdots,t^n,\cdots\}$$
is an infinite subgroupoid generated by $t$.
\end{enumerate}
\end{example}

To compute $\overline{I}$ (and so full elements or maximal elements in $\rst{I}$) effectively, 
the chain 
$$I,I^2,\cdots,I^n,\cdots$$ 
must terminate and $\overline{I}(=[I])$ must be finite. 
So some conditions are required for the finiteness of $[I]$ for a given instance $I$. Among these conditions, associativity is important one. 

Given an ER system $(\mathbb{E},\approx,\langle\,\rangle)$ and a finite instance $I$, 
Benjelloun et al. \cite{BGSWW} proved that 
$\rst{I}$ is finite if $\rst{I}$ satisfies $\eqref{eqn:ei}$, $\eqref{eqn:ec}$, $\eqref{eqn:ea}$, and $\eqref{eqn:er}$.
Converting it to the partial groupoid setting, one has:

\begin{proposition}\label{prop:pgfiniteness}
Each finitely generated partial groupoid, satisfying 
$\eqref{eqn:i}$, $\eqref{eqn:c}$, $\eqref{eqn:a}$, and $ \eqref{eqn:r}$,
is finite.
\end{proposition}

The study of finiteness conditions for semigroups is to give some conditions that can imply the finiteness of the semigroup.
Among these conditions, being finitely generated is usually required.
The finiteness of a semigroup has been studied extensively, e.g., \cite{BL, cv91, cv99}. 

An element $s$ in a semigroup (groupoid or group) is called {\em periodic} if $[s]$, the subsemigroup (subgroupoid, subgroup)
generated by $s$, is finite.
 A semigroup (groupoid or group) is called {\em periodic} if all of its elements are periodic. 
{\em The Burnside problem}, posed by William Burnside in 1902 \cite{B}, asked if a finitely generated periodic group is finite.
Golod \cite{G} provided a counter-example to the Burnside Problem. That is, there exists an infinite, finitely generated, periodic group.

The Burnside problem was subsequently extended to semigroups.
The study of whether a finitely generated periodic semigroup is finite,  is called {\em Burnside problem for semigroups}.
Morse and Hedlund gave a negative answer to Burnside problem for semigroups \cite{MH, cv99}.

{\em Burnside problem for partial groupoids} asks when a finitely generated partial periodic groupoid 
$P=[p_1,\cdots,p_n]$ is finite.

If $n=1$, then the finiteness of $P =\cup_{i=1}^{+\infty}p_1^i=\{p_1,\cdots,p_1^k,\cdots\}$
is equivalent to the periodicity of $p_1$. To keep things simple, we require each generator $p_i$ is idempotent. 

Let $([p_1,\cdots,p_n],D,\circ)$ be a finitely generated partial groupoid satisfying $\eqref{eqn:sym}$, $\eqref{eqn:cass}$,
and $\eqref{eqn:nr}$. 
By Proposition \ref{prop:tocomm},
there are a finitely generated partial groupoid 
$([q_1,\cdots,q_n],E,\cdot)$ satisfying  $\eqref{eqn:cass}$, $\eqref{eqn:c}$, and $\eqref{eqn:nr}$ and a  surjective partial
groupoid homomorphism 
$$f: ([p_1,\cdots,p_n],D,\circ)\rw ([q_1,\cdots,q_n],E,\cdot).$$
Since $\eqref{eqn:cass}$, $\eqref{eqn:c}$, and $\eqref{eqn:nr}$ imply $\eqref{eqn:i}$, $\eqref{eqn:c}$, $\eqref{eqn:a}$,
and $\eqref{eqn:r}$,
$([q_1,\cdots,q_n],E,\cdot)$ is finite by Proposition \ref{prop:pgfiniteness}.

For each $p\in [p_1,\cdots,p_n]$, by Proposition \ref{prop:tosemi}
the equivalence class $\rst{p}$, defined in Section \ref{sect:commutativegroupoid},
is a semigroup with respect to $\circ$ and $x_1\cdots x_k=x_1x_k$ for all $x_1,\cdots,x_k\in \rst{p}$.
Since for all $x,y\in \rst{p}$, $xpy=xy$ and $pxyp=p$,
for each $x\in \rst{p}$ the length of $x$, e.g., the minimum $k$ such that
$$x=x_1\cdots x_k$$
with ``(" and ``)" being inserted binarily, where $x_i\in \{p_1,\cdots,p_n\}$, must be bounded.
Hence we have:

\begin{proposition}\label{pro:psfiniteness}
Each finitely generated partial groupoid $[p_1,\cdots,p_n]$, 
satisfying $\eqref{eqn:sym}$, $\eqref{eqn:cass}$,  and $\eqref{eqn:nr}$, is finite.
\end{proposition}

In general, match and non-match labels for entity resolution is highly imbalanced, 
so the connected components in a finitely generated partial groupoid are not large.
In this paper, we do not focus on the finiteness conditions of a finitely generated partial groupoid,
which are of independent interest and are beyond the scope of the paper. 
We shall study this topic in another paper.

$\eqref{eqn:i}$, $\eqref{eqn:c}$, $\eqref{eqn:a}$, and $\eqref{eqn:r}$ are sufficient conditions for the
finiteness of a partial groupoid but not necessary as shown by the following example.
\begin{example}
Let $Q_2=\{a,b,c\}$ be the set of 3 distinct elements and define
\[
\begin{array}{ccc} 
a \star b  =  c, & b\star c =  b, &
c\star c  =  b,\\
\end{array}
\]
and no other compositions are defined. Then $(Q_2,\star)$ is a finite partial groupoid that 
does not satisfy any of $\eqref{eqn:i}$, $\eqref{eqn:c}$, $\eqref{eqn:a}$, and $\eqref{eqn:r}$.
\end{example}
\begin{proof}
\begin{description}
\item[$\eqref{eqn:i}$]
Since $c\star c=b$, $a\star a$ and  $b\star b$ are undefined, property $\eqref{eqn:i}$ dose not hold true.
\item[$\eqref{eqn:c}$]
As $a \star b =c$ and $b\star a$ is undefined, $(Q_2,\star)$ dose not satisfy property $\eqref{eqn:c}$.
\item[$\eqref{eqn:a}$]
Note that
$$(a\star b)\star c= c\star c =b
\mbox{ and } 
a\star (b\star c)=a\star b=c.$$
property $\eqref{eqn:a}$ dose not hold true.
\item[$\eqref{eqn:r}$]
Since $a\star b=c$, $a\sim b$. 
Note that $a\star b=c\nsim b$ as $c\star b$ is undefined.
Hence $(Q_2,\star)$ does not satisfy property $\eqref{eqn:r}$.
\end{description}
\end{proof}

To match the clients from multiple silos, one may want to identify the client records that have the maximal information,
in the merge closure.
However, given a database with certain operations, e.g., selection, sorting, join, 
we may want to find the smallest set of tables in the database,
which are irreducible generators for the database.

Once a finitely generated partial groupoid or semigroup, is finite and has a transitive order,
one can sort, select, and query over the partial groupoid using the results for transitive partial order in \cite{dkmrv}. 
Hence we consider entity resolution as
\begin{equation}\label{eqn:erpg}
\mbox{\em Sorting, selecting, and querying elements over a finitely generated partial groupoid}\tag{$ER_{pg}$}
\end{equation}
$\mbox{\em which is finite and partially ordered.}$

In real life, entities are not isolated but connected by their relationships,
e.g., customers are grouped by their social media relationships,
machine learning models are based on certain data sets related and machine learning models, on the other hand,
output data. When merging the entities matched, it is desirable that the relationships detected between entity objects are preserved.

An ER system can be modeled as a finitely generated partial groupoid $([p_1,\cdots,p_n],D,\circ)$ with certain properties that
guarantee: the finiteness of $[p_1,\cdots,p_n]$ and the existence of a natural partial order.
Assume that there is a set of initial relationships, represented by a directed graph $G_0$, a subgraph of the complete graph with the nodes from $[p_1,\cdots,p_n]$. The entity resolution with the relationships being preserved can then be modeled by
$([p_1,\cdots,p_n],(G_i),D,*)$, where $G_i$ is a sequence of the directed graph of the relationship, starting from $G_0$,
$*$ is the extension $\circ$, which includes merging elements in $[p_1,\cdots,p_n]$
and updating the relationship graph $G_{i-1}$ to obtain $G_i$,  
so that the previous relationships in $G_{i-1}$ are preserved.

\section{Conclusions}\label{chapter:conclusions}
An ER system $(\mathbb{E},\approx, \langle\,\rangle)$ is the same as a partial groupoid
$(P,D,\circ)$, where the binary operation $\circ$ is given by 
the merge function $\langle\,\rangle$ and the partiality $D$ of the operation $\circ$ is given by the match function $\approx$.
Given an instance $I$, the merge closure of $I$ turns out  to be the finitely generated partial subgroupoid $[I]$.

The properties: $\eqref{eqn:i}$, $\eqref{eqn:wc}$, $\eqref{eqn:a}$, 
$\eqref{eqn:r}$,  $\eqref{eqn:lub}$, and $\eqref{eqn:comp}$, match and merge functions share, and their relationships
on a partial groupoid, were studied. The natural partial order on a partial
groupoid was introduced when the partial groupoid satisfies $\eqref{eqn:i}$ and $\eqref{eqn:cass}$.
Given a partial order on a partial groupoid, 
$\eqref{eqn:lub}$ and $\eqref{eqn:comp}$ are equivalent to
$\eqref{eqn:i}$, $\eqref{eqn:wc}$, $\eqref{eqn:a}$, 
and $\eqref{eqn:r}$ and the partial order must be the natural one we defined
when the domain of the partial operation is reflexive.

The partiality of a partial groupoid was reduced using connected components and clique covers of its domain graph
and a noncommutative partial groupoid was mapped to a commutative one homomorphically, 
when it has the partial semigroup like structures.

In a finitely generated partial groupoid $(P,D,\circ)$ without any conditions required, 
the entity resolution $\eqref{eqn:er0}$ we concerned was defined to be the full elements in $P$.
If $(P,D,\circ)$ satisfies $\eqref{eqn:i}$ and $\eqref{eqn:cass}$,
then the entity resolution $\eqref{eqn:er1}$ was the maximal elements in $P$,
which are full elements and form the entity resolution $ER(P)$, defined in \cite{BGSWW}.
Furthermore, in the case, we considered entity resolution as $\eqref{eqn:erpg}$:
``sorting, selecting, and querying the elements in a finitely generated partial groupoid."



\end{document}